\documentclass{article}

\usepackage{arxiv}

\usepackage[utf8]{inputenc} % allow utf-8 input
\usepackage[T1]{fontenc}    % use 8-bit T1 fonts
\usepackage{hyperref}       % hyperlinks
\usepackage{url}            % simple URL typesetting
\usepackage{booktabs}       % professional-quality tables
\usepackage{amsfonts}       % blackboard math symbols
\usepackage{nicefrac}       % compact symbols for 1/2, etc.
\usepackage{microtype}      % microtypography
\usepackage{latexsym}
\usepackage{graphicx}
\usepackage{mathptmx}

\usepackage{hyperref}       % hyperlinks
\usepackage{url}            % simple URL typesetting
\usepackage{booktabs}       % professional-quality tables
\usepackage{amsfonts}       % blackboard math symbols
\usepackage{nicefrac}       % compact symbols for 1/2, etc.
\usepackage{microtype}      % microtypography
\usepackage{lipsum}		% Can be removed after putting your text content
\usepackage{footnote}
\usepackage{setspace}
\usepackage{graphicx}
\usepackage{doi}
\usepackage{endnotes}

\def\E{{\mathbb E}}

\providecommand{\keywords}[1]
{
  \small	
  \textbf{\textit{Keywords---}} #1
}
\usepackage{mathtools}

\usepackage{colortbl}
\usepackage{comment}
\usepackage{graphicx}
\usepackage{bbm}
\usepackage{subcaption}
\usepackage{siunitx}
\usepackage{appendix}
\usepackage{enumitem}
\usepackage{amsmath}
\usepackage{amsfonts}
\usepackage{amssymb}
\usepackage{amsthm}
\usepackage{multirow}
\usepackage{chemformula}
\usepackage{cleveref}
\usepackage{wrapfig}
\usepackage[ruled]{algorithm2e}

\usepackage{graphicx}
\usepackage[round]{natbib}
\usepackage{doi}

\newtheorem{theorem}{Theorem}

\newtheorem{lemma}{Lemma}

\newtheorem{proposition}{Proposition}

\title{Variance Reduction based Partial Trajectory Reuse to Accelerate Policy Gradient Optimization}

\author{Hua Zheng, Wei Xie\thanks{Corresponding author: w.xie@northeastern.edu}\\[12pt]
Department of Mechanical and Industrial Engineering\\
Northeastern University\\
360 Huntington Ave\\
Boston, MA 02115, USA
}

\date{}

\renewcommand{\headeright}{Preprint}
\renewcommand{\shorttitle}{\textit{Zheng and Xie}}

\hypersetup{
pdftitle={A template for the arxiv style},
pdfsubject={q-bio.NC, q-bio.QM},
pdfauthor={David S.~Hippocampus, Elias D.~Striatum},
pdfkeywords={First keyword, Second keyword, More},
}

\begin{document}
\maketitle

\begin{abstract}
Built on our previous study on green simulation assisted policy gradient (GS-PG) focusing on trajectory-based reuse, in this paper, we consider infinite-horizon Markov Decision Processes and create a new importance sampling based policy gradient optimization approach to support dynamic decision making. The existing GS-PG method was designed to learn from complete episodes or process trajectories, which limits its applicability to low-data situations and flexible online process control. To overcome this limitation, the proposed approach can selectively reuse the most related partial trajectories, i.e., the reuse unit is based on per-step or per-decision historical observations. In specific, we create a mixture likelihood ratio (MLR) based policy gradient optimization that can leverage the information from historical state-action transitions generated under different behavioral policies. The proposed variance reduction experience replay (VRER) approach can intelligently select and reuse most relevant transition observations, improve the policy gradient estimation, and accelerate the learning of optimal policy. Our empirical study demonstrates that it can improve optimization convergence and enhance the performance of state-of-the-art policy optimization approaches such as actor-critic method and proximal policy optimizations. 
\end{abstract}

\vspace{0.1in}
% keywords can be removed

\vspace{0.3in}
\section{Introduction}
In recent years, various policy optimization approaches are developed to solve reinforcement learning (RL) and process control problems.
%applicable to solve reinforcement learning (RL) and process control problems \citep{%schulman2015trust,
%schulman2017proximal}, such as healthcare \citep{yu2021reinforcement,zheng2021reinforcement} and biomanufacturing \citep{zheng2021policy}. 
They often consider parametric policies and search for optimal solution through stochastic policy gradient approach. Historical observations can be reused to improve optimization convergence, especially in low-data situations.
% This category of RL algorithm perform a search in a parametric policy space by following the gradient of the utility function estimated by means of a batch of trajectories collected from the environment 
According to the base unit of observations to reuse, %utilized to compute the gradient, 
policy gradient %optimization 
algorithms can be classified into episode-based and step-based approaches \citep{metelli2020importance}. 
%\textcolor{red}{Episode–based approaches, such as REINFORCE \citep{williams1992simple,zheng2021policy}, are also known as Monte Carlo approaches while step–based approaches (also known as per-decision approaches), derived from the Policy Gradient Theorem \citep{Sutton1999}, include deep deterministic policy gradient (DDPG) \citep{lillicrap2016continuous}, trust region policy optimization (TRPO) \citep{schulman2015trust}, proximal policy optimization (PPO) \citep{schulman2017proximal} and many other extensions. The main difference between them is that step–based approaches, use bootstrapping to approximate the action-value function and policy gradient while episode–based (or Monte Carlo) approaches take sample average of gradients of each episode to accomplish this. The latter, in some cases, requires to have a finite horizon.}
%To reuse historical samples, 
Episode-based approaches perform importance sampling (IS) on full trajectories accounting for the distributional difference induced by target and %sampling/
behavior policies. The
importance weight is built on the cumulative product of likelihood ratios (LR) of state-action transitions occurring within each trajectory. %the products of policy ratios for all transitions within a trajectory.
%\citep{metelli2020importance,zheng2021green}. 
%The cumulative product of likelihood ratios across all transitions in a trajectory 
This can lead to extremely high variance, especially for those problems with long planning horizon \citep{andradottir1995choice,schlegel2019importance}. %As a result, the LR based policy gradient estimator inevitably suffers from high variance and thus becomes impractical. 
%In addition, 
\textit{Thus, the trajectory-based reuse strategy is not be applicable to many applications, such as personalized cell therapy manufacturing, with: (1) small amount of data; (2) complex state-action transition model and long planing horizon; and (3) requiring real-time flexible process control.}

Instead of reusing entire historical trajectories, we need to have an intelligent and flexible strategy that can select and reuse the most related parts of trajectory which can change for different random scenarios. For example, during cell therapy manufacturing, the cell metabolic state can evolve with time during its life cycle and also therapeutic cells can have metabolic shift under heterogeneous culture conditions. To improve prediction and guide real-time process control, we can reuse historical step-based observations that have cell metabolic state and bioprocessing dynamics similar to the target distribution.
\textit{Therefore, step-based policy gradient algorithms can take state-action transitions as the base unit of observations to reuse  \citep{metelli2020importance} and overcome the limitations of the trajectory-based reuse strategy.}

In this paper, we focus on step-based policy gradient for infinite-horizon Markov Decision Processes (MDPs).
That means we update policy parameters per step (or mini-batch of steps), which requires a single state-action transition LR for each historical observation %per sample 
to account for the difference in
state-occupancy measures or 
stationary state distribution induced by different target and behavior policies. 
Various step-based policy gradient algorithms have been proposed during recent years.
%One major theoretical challenge %barrier that impedes the application of LR in step-based algorithms is that
%, unlike episode-based approaches, the LR involves unknown transition probabilities and thus cannot be simplified in an analytical form. 
The studies in distribution correction (DICE)  \citep{NEURIPS2019_cf9a242b,yang2020off} %zhang2020gradientdice
provide ways to estimate these state occupancy ratios in RL. 
%It paves the way for a step-based %version of variance reduction experience replay for policy gradient  
%Through coupled with mixture likelihood ratio (MLR), it can improve the policy gradient estimation  \citep{zheng2021green}.
Proximal policy optimization (PPO), as one of the most popular step-based policy gradient approaches \citep{schulman2017proximal}, uses a clipped surrogate objective to control incentives for the new candidate policy to get far from the old policy and thus avoid parameter updates too much at one step. Actor-Critic algorithm, as a classic and theoretically solid policy optimization framework, jointly optimizes the value function (critic) and the policy (actor); see for example \cite{bhatnagar2009natural}. 

%In addition, \cite{degris2012off} generalize the algorithm in off-policy setting by enabling a target policy to be learned by using data from behavior policies.

%Motivated by the high variance problem discussed above,
In our previous study \citep{zheng2021green}, we created a new experience replay approach called variance reduction experience replay (VRER) for 
%and investigated its applicability to 
episode-based policy optimization. This approach can automatically select the most relevant historical trajectory episodes based on a comparison of gradient variance between historical episodes and current episodes, i.e., episodes collected by following the candidate policy. Then the selected historical trajectories are used to improve policy gradient estimation through multiple importance sampling techniques.
Our theoretical and empirical studies have showed that the VRER based policy gradient estimator can improve sample efficiency and lead to a superior performance in convergence.

Build on \cite{zheng2021green}, in this paper, we create a partial trajectory based VRER approach for infinite horizon MDPs. It can select and reuse the most relevant historical observations on state-action transitions to improve policy gradient estimation. 
The proposed VRER approach is general and it can be integrated into various stochastic policy gradient approaches to improve optimization convergence. In the paper, we provide an algorithm to utilize it to enhance two state-of-the-art policy optimization algorithms, including Actor-Critic algorithm and PPO.

Therefore, the key contributions of this study include: (1) create the VRER based policy optimization that can selectively reuse the most relevant historical transitions or partial trajectory observations; (2) develop multiple importance sampling (MIS) based off-policy actor-critic method;   %from a episode-based approach to a step-based approach; 
and (3) analytically and empirically show that the proposed VRER based policy gradient approach %(VRER-PG)
can improve the policy gradient estimation, speed up the optimal convergence for RL problems, and support real-time process control. 

The paper is organized as follows. We start in Section \ref{sec: problem description} by introducing the notation and
basics about RL policy optimization, importance sampling (IS), and multiple importance sampling. Then we propose the mixture likelihood ratio (MLR) based policy gradient estimation in Section~\ref{sec:MLR}. We create a selection rule that allows us to automatically select the most relevant historical transitions to improve the policy gradient estimation accuracy in Section~\ref{sec:algorithms}. We further show how this selection strategy can be customized into the policy gradient optimization and introduce the variance reduction experience replay with the resulting algorithms. We conclude this paper with
the comprehensive empirical study on the proposed framework in Section~\ref{sec: empirical study}. The implementation of our algorithm can be found at \url{https://github.com/zhenghuazx/vrer_policy_optimization}.

% Intuitively, this variance increases proportionally to the distance between the behavioral and the target policy; thus, the estimate is more reliable when two policies are close enough. 

% ------------------------
\vspace{-0.1in}
\section{Problem Description}
\label{sec: problem description}

\vspace{-0.1in}

%\subsection{Formulation of Infinite Horizon MDP}
We study reinforcement learning and process control problems in
which an agent acts on a complex stochastic system %in a stochastic environment 
by sequentially choosing actions over a sequence of time steps in
order to maximise a cumulative reward.
We formulate the problem of interest as an infinite-horizon Markov decision process (MDP) specified by $(\mathcal{S},\mathcal{A}, r, \mathbb{P}, \pmb{s}_1)$, where $\mathcal{S}$ is the state space, $\mathcal{A}$ is the action space, and a reward function is denoted by $r: \mathcal{S}\times \mathcal{A}\rightarrow\mathbb{R}$. An initial state distribution is specified by the density $p_1(\pmb{s}_1)$. The stationary state transition model $p(\pmb{s}_{t+1}|\pmb{s}_t,\pmb{a}_t)$ satisfies the Markov property $p(\pmb{s}_{t+1}|\pmb{s}_t,\pmb{a}_t,\ldots,\pmb{s}_1,\pmb{a}_1)=p(\pmb{s}_{t+1}|\pmb{s}_t,\pmb{a}_t)$. 
The system starts at an initial state $\pmb{s}_1$ at time $t=1$ drawn from $p_1(\pmb{s}_1)$.
At time $t$, the agent observes the state $\pmb{s}_t \in \mathcal{S}$, takes an action $\pmb{a}_t \in \mathcal{A}$ by following a parametric policy distribution, denoted by $\pi(\pmb{s}_t|\pmb{a}_t;\pmb\theta)$ specified with parameters $\pmb\theta$, %\in \mathbb{R}^d$,
and receives %a feedback in form of 
a reward $r_{t}(\pmb{s}_t,\pmb{a}_t) \in \mathbb{R}$.
The future return is the total discounted reward, denoted by % from time-step $t$ onwards,
$r_t^\gamma = \sum_{t^\prime=t}^\infty \gamma^{t^\prime-t}r(\pmb{s}_{t^\prime},\pmb{a}_{t^\prime})$, where $\gamma\in(0,1)$ denotes the discount factor.

Suppose that the policy $\pi_{\pmb\theta}$ is continuous and differentiable with respect to its parameters $\pmb\theta$. For each candidate policy specified by $\pmb\theta$, the state value function $V^\pi(\pmb{s})$
and the Q-function $Q^\pi(\pmb{s}, \pmb{a})$ are defined to be the expected total discounted reward-to-go, i.e.,
\begin{eqnarray}
  V^\pi(\pmb{s})&= &\E[r_1^\gamma| \pmb{s}_1=\pmb{s};\pi]
    =\E\left[\left.\sum_{t=1}^\infty \gamma^{t-1}r(\pmb{s}_{t},\pmb{a}_{t})\right| \pmb{s}_1=\pmb{s};\pi\right],
    \nonumber  \\ 
    Q^\pi(\pmb{s},\pmb{a}) &= &\E[r_1^\gamma| \pmb{s}_1=\pmb{s},\pmb{a}_1=\pmb{a};\pi]=\E\left[\left.\sum_{t=1}^\infty \gamma^{t-1}r(\pmb{s}_{t},\pmb{a}_{t})\right| \pmb{s}_1=\pmb{s}, \pmb{a}_1=\pmb{a};\pi\right].
    \nonumber 
\end{eqnarray}
 The agent’s 
 goal is to find an optimal
policy that maximises the cumulative discounted reward, denoted by %the performance objective
$J(\pi)=\E[\sum_{t=1}^\infty \gamma^{t-1}r(\pmb{s}_{t},\pmb{a}_{t})|\pi].$ 
For any feasible policy $\pmb\theta$, %\in \mathbb{R}^d$, 
we assume that the Markov chains, i.e., $\{\pmb{s}_t\}_{t\geq \infty}$ and $\{\pmb{s}_t,\pmb{a}_t\}_{t\geq \infty}$, are irreducible and aperiodic.
We denote the improper discounted state distribution as
%with stationary discounted state distribution
$$d^\pi(\pmb{s})=%(1-\gamma)
\int _\mathcal{S}
\sum_{t=1}^\infty\gamma^{t-1}p(\pmb{s}_1)p(\pmb{s}_{t}=\pmb{s}|\pmb{s}_1;\pi)d \pmb{s}_1.
$$
% $d^\pi(\pmb s)=\limit_{t\rightarrow \infty}P(\pmb{s}_t=\pmb{s}|\pmb{s}_0,\pi_\pmb\theta)$ is the probability that $\pmb{s}_t$ when starting from $\pmb{s}_0$ and following policy $\pi_\theta$ for t steps.
Then we can write the policy optimization problem with
the performance objective as an expectation,
\begin{equation} \label{eq: objective}
  \max_{\pmb{\theta}} ~  J(\pmb\theta)
  =  \E \left[ \left. 
  \sum_{t=1}^\infty \gamma^{t-1}r(\pmb{s}_{t},\pmb{a}_{t})
  \right|\pi \right]
    =\int d^\pi(\pmb{s}) \int \pi_{\pmb\theta}(\pmb{a}|\pmb{s})r(\pmb{s},\pmb{a})d\pmb{s}d\pmb{a} = \E_{\pmb{s}\sim d^\pi(\pmb{s}),\pmb{a}\sim \pi_{\pmb\theta}(\pmb{a}|\pmb{s})}[r(\pmb{s},\pmb{a})],
\end{equation}
where $\E_{\pmb{s}\sim d^\pi(\pmb{s}),
\pmb{a}\sim \pi_{\pmb\theta}(\pmb{a}|\pmb{s})}[\cdot]$ denotes the expected value with
respect to the discounted state distribution $d^\pi(\pmb{s})$ and the policy distribution $\pi_{\pmb\theta}(\pmb{a}|\pmb{s})$. We denote 
%the Markov chains $\{\pmb{s}_t,\pmb{a}_t\}$ are irreducible and aperiodic with 
the stationary probability function of state-action pair by 
\begin{equation}
\rho_{\pmb\theta}(\pmb{s},\pmb{a})=\pi_{\pmb\theta}(\pmb{a}|\pmb{s})d^\pi(\pmb{s}).
\nonumber 
\end{equation}
To simplify notation, we superscript the value function $V^\pi(\pmb{s})$ and the advantage function, denoted by $A^\pi(\pmb{s},\pmb{a})$ that will be defined in Section~\ref{subsec:stochasticPG}, by $\pi$ rather than $\pi_{\pmb\theta}$. % but we will keep iteration indices if they are involved in the notations of policy parameters.

% ---------------------

\subsection{Stochastic Policy Gradient Estimation
}
\label{subsec:stochasticPG}
Policy gradient optimization is perhaps the most popular class of %continuous action 
RL algorithms designed to solve the optimization problem~\eqref{eq: objective}. At each $k$-th iteration, we can iteratively update the policy parameters,
\begin{equation} 
\label{eq: policy gradient update}
    \pmb\theta_{k+1} \leftarrow \pmb\theta_k + \eta_k \widehat{\nabla J}(\pmb\theta_k),
    \nonumber 
\end{equation}
where $\eta_k$ is learning rate or step size and $\widehat{\nabla J}(\pmb{\theta}_k)$ is an estimator of policy gradient $\nabla J(\pmb{\theta}_k)$.
For notational convenience, $\nabla$ denotes the gradient with respect to policy parameters $\pmb{\theta}$ unless specified otherwise.
% The basic idea behind these algorithms is to adjust the parameters $\pmb\theta$ of the policy in the direction of the performance gradient.
Under regularity conditions, \textit{Policy Gradient Theorem} \citep{sutton1999policy} reformulates the policy gradient as
\begin{align}
    \nabla J(\pmb{\theta}) = \int d^\pi(\pmb{s}) \int \nabla\pi_{\pmb\theta}(\pmb{a}|\pmb{s})Q^\pi(\pmb{s},\pmb{a})d\pmb{s}d\pmb{a} %\nonumber\\     &= 
    =\E_{\pmb{s}\sim d^\pi(\pmb{s}),\pmb{a}\sim \pi_{\pmb\theta}(\pmb{a}|\pmb{s})}[\nabla \log \pi_{\pmb\theta}(\pmb{a}|\pmb{s})Q^\pi(\pmb{s},\pmb{a})].
    \label{eq: policy gradient}
\end{align}
This theorem has an important practical value because it reduces the computation of the performance gradient to a
simple expectation \citep{silver2014deterministic}. %The policy gradient theorem has been used to derive a variety of policy gradient algorithms. 
By applying sample average approximation (SAA) on the expectation in (\ref{eq: policy gradient}), we have the \textit{naive policy gradient (PG) estimator},
\begin{equation}
    \widehat{\nabla J}^{PG}_k \equiv
    \widehat{\nabla J}^{PG}(\pmb{\theta}_k) 
=\frac{1}{n} \sum_{j=1}^n g_k\left(\pmb{s}^{(k,j)},\pmb{a}^{(k,j)}\right) 
    \label{eq.PG-estimator}
    \nonumber 
\end{equation}
where $n$ is the number of replications. The \textit{scenario-based policy gradient estimate} at $\pmb{\theta}_k$ is represented as,
\begin{equation}
 g_k\left(\pmb{s},\pmb{a}\right) \equiv g\left(\pmb{s},\pmb{a}|\pmb\theta_k\right)= Q^{\pi}(\pmb{s},\pmb{a})\nabla \log \pi_{\pmb{\theta}_k}
 \left({\pmb{a}\left|\pmb{s}\right.}\right).
 \label{eq.scenariobasedGradient}
 \nonumber 
 \end{equation}
 
A widely used variation of policy gradient \eqref{eq: policy gradient} is to subtract a baseline value from the return $Q^\pi(\pmb{s},\pmb{a})$
 to reduce the variance of gradient estimation while keeping the unbiased property. A common baseline is to subtract a value function $V^\pi(\pmb{s})$; see \cite[Lemma 2]{bhatnagar2009natural}. Then we have a new unbiased policy gradient estimator with lower variance,
 \begin{align}
    \nabla J(\pmb{\theta})
    &= \E_{\pmb{s}\sim d^\pi(\pmb{s}),\pmb{a}\sim \pi_{\pmb\theta}(\pmb{a}|\pmb{s})}\left[\nabla \log \pi_{\pmb\theta}(\pmb{a}|\pmb{s})\left(Q^\pi(\pmb{s},\pmb{a})-V^\pi(\pmb{s})\right)\right].
    \label{eq: policy gradient with advantage}
\end{align}
 The difference $A^\pi(\pmb{s},\pmb{a}) \equiv Q^\pi(\pmb{s},\pmb{a})-V^\pi(\pmb{s})$ is called \textit{advantage}. It intuitively measures the extra reward that an agent can obtain by taking that a particular action $\pmb{a}$. This leads to the ``vanilla" policy gradient estimator,
 \begin{equation}
 g_k\left(\pmb{s},\pmb{a}\right) \equiv g\left(\pmb{s},\pmb{a}|\pmb\theta_k\right)= A^{\pi_{\pmb\theta_k}}(\pmb{s},\pmb{a})\nabla \log \pi_{\pmb{\theta}_k}
 \left({\pmb{a}\left|\pmb{s}\right.}\right).
 \label{eq: vannila scenario based gradient}
 \nonumber 
 \end{equation}

According to the Bellman equation, i.e., $Q^{\pi}(\pmb{s},\pmb{a})=r(\pmb{s},\pmb{a})+\gamma\E_{\pmb{s}^\prime\sim p(\pmb{s}^\prime|\pmb{s},\pmb{a})}[V^{\pi}(\pmb{s}^\prime)]$, the advantage function can be expressed as
\begin{equation}\label{eq: advantage}
    A^{\pi}(\pmb{s},\pmb{a})=r(\pmb{s},\pmb{a})+\gamma\E_{\pmb{s}^\prime\sim p(\pmb{s}^\prime|\pmb{s},\pmb{a})}[V^{\pi}(\pmb{s}^\prime)]-V^{\pi}(\pmb{s}).
\end{equation}
Let $\hat{V}(\pmb{s})$ denote an unbiased estimator of value function at state $\pmb{s}$. Then, for any given \textit{state-action transition observation}, denoted by $(\pmb{s},\pmb{a},\pmb{s}^\prime)$, the temporal difference (TD) error, i.e.,
\begin{equation} \label{eq: advantage estimate}
    \delta(\pmb{s},\pmb{a},\pmb{s}^\prime)=r(\pmb{s},\pmb{a})+\gamma \hat{V}(\pmb{s}^\prime)-\hat{V}(\pmb{s})
    \nonumber 
\end{equation}
is an unbiased estimator of the
advantage \eqref{eq: advantage}; see \cite[Lemma 3]{bhatnagar2009natural}.

In a nutshell, estimating the advantage function requires a set of observations $\{(\pmb{s}_t,\pmb{a}_t,\pmb{s}_{t+1}, r_t)\}$ while the policy gradient estimate %\eqref{eq: vannila scenario based gradient} 
involves the estimated advantage function $A^\pi(\pmb{s},\pmb{a})$ and the estimated score function $\nabla \log \pi_{\pmb{\theta}_k}\left(\pmb{a}|\pmb{s}\right)$ at state-action pairs $\{(\pmb{s}_t, \pmb{a}_t)\}$. We will discuss the policy gradient estimation and advantage estimation separately. From Section~\ref{subsec: MIS} to \ref{subsec:off-policy policy gradient estimator}, we will focus on the multiple importance sampling based policy optimization and policy gradient estimation, while supposing that an unbiased estimator of advantage function is given. Then in Section~\ref{subsec:model free GreenSimulation_PolicyGradient}, we will show how to estimate the advantage function through temporal difference learning.
\textit{In this paper, we will focus on creating a selective historical transition replay approach in order to improve the policy gradient estimation.}

% ----------------------
\subsection{Importance Sampling and Multiple Likelihood Ratio}\label{subsec: MIS}
%to Improve Policy Gradient Estimation}

In this section, we describe how to utilize  important sampling (IS) and multiple likelihood ratio (MLR) to improve the policy gradient estimation. Denote the state-action input pair as $\pmb{x} \equiv (\pmb{s},\pmb{a})$. Let $\rho_i(\pmb{x})=\rho_{\pmb\theta_i}(\pmb{s},\pmb{a})$ represent the stationary sampling generative distribution at the $i$-th episode obtained under a policy specified by $\pmb{\theta}_i$. For any candidate policy specified by $\pmb{\theta}$, let  $\rho(\pmb{x})=\rho_{\pmb\theta}(\pmb{s},\pmb{a})$ denote the target distribution or likelihood. We are interested in estimating the expected gradient $\nabla J(\pmb\theta)=\E_{\rho(\pmb{x})}\left[g\left(\pmb{x}\right)\right]
=\E_{\rho_{\pmb\theta}(\pmb{s},\pmb{a})}\left[g\left(\pmb{s},\pmb{a}|\pmb\theta\right)\right]$.

When the historical samples generated from the \textit{sampling distribution} $\rho_i$ are selected and reused to estimate the candidate policy gradient $\nabla J(\pmb\theta_k)$
% $\E_\rho\left[g(\pmb{x})\right]$ 
under the \textit{target distribution} $\rho_k$, the importance sampling estimator \citep{andradottir1995choice,rubinstein2016simulation} %hesterberg1988advances
corrects the sampling distribution
with the importance weight or likelihood ratio defined as $f(\pmb{x})={\rho_k(\pmb{x})}/{\rho_i(\pmb{x})}$, i.e.,
\begin{equation} 
\label{eq: importance sampling estimator}
    \widehat{\nabla J}(\pmb\theta_k)
    =\frac{1}{n}\sum^n_{j=1}
    f\left(\pmb{x}^{(i,j)}\right)
    g_k\left(\pmb{x}^{(i,j)}\right),
    \nonumber 
\end{equation}
\noindent where  $\pmb{x}^{(i,j)}\stackrel{i.i.d.}\sim \rho_i (\pmb{x})$ with $j=1,2,\ldots,n$.
For simplification, we allocate a constant number of replications (i,e., $n$) for each visit at $\pmb{\theta}$.
We assume $\rho_i(\pmb{x})>0$
whenever $\rho_k(\pmb{x})g_k(\pmb{x})\neq 0$. % for any $i\in\mathbb{Z}^+$.  
This estimator is unbiased,
\begin{equation} 
\label{eq: individual likelihood unbiaed estimator}
    \E_{\rho_i}\left[\widehat{\nabla J}(\pmb\theta_k) \right]=\int\frac{\rho_k(\pmb{x})}{\rho_i(\pmb{x})}\rho_i(\pmb{x})g_k(\pmb{x}) \mbox{d}\pmb{x}=\int \rho_k\left(\pmb{x}\right)g_k\left(\pmb{x}\right) \mbox{d}\pmb{x}=\nabla J(\pmb\theta_k).
    \nonumber 
\end{equation}

%Although this importance estimator is unbiased, the importance sampling estimate is often of unnecessarily high variance. 
However, the likelihood ratio ${\rho_k(\pmb{x})}/{\rho_i(\pmb{x})}$ can be extremely large or small at sample $\pmb{x}$.
%when we have very small/large likelihood $\rho_i(\pmb{x})$ and large/small $\rho_k(\pmb{x})$. %That means a large difference between the sampling and target distributions.
Without any bound on the likelihood ratio ${\rho_k(\pmb{x})}/{\rho_i(\pmb{x})}$, the importance sampling estimator $\widehat{\nabla J}$
% $\nabla_{\pmb{\theta}}\widehat{\mu}^{ILR}_{k}$ 
can have inflated variance, which is typically induced by large difference between target and proposal distributions.
 Inspired by the BLR-M metamodel \citep{FengGreenSim2017} and multiple importance sampling \citep{veach1995optimally}, we utilize the mixture likelihood ratio (MLR) method to address this issue. It has a mixture sampling  distribution, denoted by $\ell_M(\pmb{x}) \equiv \sum_{i\in{U}}\alpha_i \rho_i(\pmb{x})$ with $\sum_{i\in{U}}\alpha_i=1$, %where $\alpha_i$ denotes the weight for the $i$th proposal distribution. 
 composed of multiple distribution components, denoted by $\{\rho_i(\pmb{x}):i\in{U}\}$, where $U$ represents the reuse set. Thus, the MLR can avoid the limitations induced by using a single proposal distribution $\rho_i(\pmb{x})$ and reduce the variance inflation issue.

Given the samples generated from those sampling distributions, denoted by $\{\pmb{x}^{(i,j)}: i\in {U} ~\mbox{and}~ j = 1,2,\ldots,n\}$, the MLR estimator, as stratified
sampling from the mixture distribution $\ell_M(\pmb{x})$, becomes
\begin{equation} 
\label{eq: MLR estimator} 
    \widehat{\nabla J}_k^{MLR}= \frac{1}{|{U}|}\sum_{i\in{U}}\frac{1}{n}\sum^{n}_{j=1}f_M\left(\pmb{x}^{(i,j)}\right)
    g_k\left(\pmb{x}^{(i,j)}\right)
    ~~~\mbox{with}~~~
     f_M(\pmb{x}) = \frac{\rho_k(\pmb{x})}{\ell_M(\pmb{x})}
    = \frac{\rho_k(\pmb{x})}{\sum_{i\in{U}}\alpha_i \rho_i(\pmb{x})}.
\end{equation}
\textit{To have a unbiased MLR estimator, the weight is selected to be the proportion of historical sample size generated from each proposal distribution component $\rho_i(\pmb{x})$, i.e., $\alpha_i = \frac{n}{\sum_{i\in{U}} n}$.}
Suppose there are $n$ historical samples generated from each distribution. Then, we allocate equal weight on $\rho_i(\pmb{x})$, i.e., $\alpha_i={1}/{|{U}|}$ for $i\in{U}$, where $|\cdot|$ denotes set cardinality. This MLR estimator is unbiased \citep{veach1995optimally},
\begin{eqnarray}\small
%\lefteqn{
\E\left[\widehat{\nabla J}_k^{MLR}\right]
=\E\left[\frac{1}{|{U}|}\sum_{i\in {U}}\frac{1}{n}
\sum^{n}_{j=1}\frac{\rho_k\left(\pmb{x}^{(i,j)}\right)}{\sum_{i\in{U}}\alpha_i \rho_i\left(\pmb{x}^{(i,j)}\right)}
g_k\left(\pmb{x}^{(i,j)}\right)
\right]
%\nonumber}\\
=\frac{1}{|{U}|}\sum_{i\in{U}}\int\frac{\rho_k(\pmb{x})}{\frac{1}{|{U}|}\sum_{i\in{U}} \rho_i(\pmb{x})}
g_k\left(\pmb{x}\right)\rho_i(\pmb{x})
\mbox{d}\pmb{x} %\nonumber\\
=\nabla J(\pmb\theta_k)\nonumber.
\end{eqnarray}

The major advantage of the MLR estimator,
compared with the standard IS, is higher sample-efficiency and lower gradient estimation variance. Since we always include the transitions generated in the current iteration, the mixture likelihood ratio $f_M(\pmb{x})$ in (\ref{eq: MLR estimator}) reaches its max value when the likelihood $\rho_i(\pmb{x})=0$ for all remaining sampling distributions with $i\in U$. %visited in previous iterations. %$i=1,2,\ldots,k-1$. 
Thus, this mixture likelihood ratio is bounded, i.e., $f_M(\pmb{x}) \leq |{U}|$, which can control the policy gradient estimation variance inflation issue. 
\textit{In this way, the mixture likelihood ratio puts higher weight on the samples that are more likely to be generated by the target distribution $\rho(\pmb{x})$ without assigning extremely large weights on the others.}

\section{Mixture Likelihood Ratio Assisted Policy Optimization}
\label{sec:MLR}
Given a set of historical samples collected under different stationary distributions and behavior policies, the off-policy strategy is used to find the optimal policy maximizing the expected return.
One can reuse the past samples to improve the policy gradient estimation through MLR. 
% Reusing historical samples requires to estimate the policy gradient off-policy from transitions generated from a different policy (a.k.a. behavior policy). 
Let $M_k$ denote the set of all policies (i.e., actor) and value functions (i.e., critic) that have been visited until the beginning of the $k$-th iteration. Let ${U}_k$ be \textit{a reuse set} including the indices of model candidates
whose transitions are selected and reused for estimating the policy gradient $\nabla J(\pmb\theta_k)$. Denote its cardinality as $|{U}_k|$. 
For discussions in this section, suppose that ${U}_k$ is given.
We will present how to determine the reuse set ${U}_k$ to improve the policy gradient estimation accuracy in Section~\ref{sec:algorithms}.

\subsection{Off-policy Policy Gradient Estimation}
\label{subsec:off-policy policy gradient estimator}

%In an off-policy setting, an agent learns about a policy different from the one it is executing and thus training samples are collected according to the sampling (behavior) policies rather than the target policy (the policy that we try to optimize for). 
Compared to on-policy alternatives, off-policy approaches do not require full trajectories and they can reuse the selected past transitions %or episodes
(``experience replay") to improve the sample efficiency. %  for much better sample efficiency. 
To use the past observations, we modify the policy gradient \eqref{eq: policy gradient with advantage} such that the mismatch between the sampling distribution $\rho_{\pmb\theta_i}(\pmb{s},\pmb{a})$ and the target distribution $\rho_{\pmb\theta_k}(\pmb{s},\pmb{a})$ is compensated by importance sampling estimator \eqref{eq: individual likelihood unbiaed estimator}, i.e.,
 
 \begin{equation} 
\label{eq: off-policy policy gradient}
    \nabla J(\pmb\theta_k)=\E_{\rho_{\pmb\theta_i}}\left[f(\pmb{s},\pmb{a})g(\pmb{s},\pmb{a})\right]=\E_{\rho_{\pmb\theta_i}}\left[\frac{\rho_{\pmb\theta_k}(\pmb{s},\pmb{a})}{\rho_{\pmb\theta_i}(\pmb{s},\pmb{a})}A^{\pi _{\pmb{\theta}_k}}(\pmb{s},\pmb{a})\nabla \log \pi_{\pmb{\theta}_k}
 \left({\pmb{a}\left|\pmb{s}\right.}\right)\right].
\end{equation}

Let $g_k(\pmb{s},\pmb{a})=\nabla \log \pi_{\pmb\theta_k}(\pmb{a}|\pmb{s})A^{\pi_{\pmb\theta_k}}(\pmb{s},\pmb{a})$. We can obtain an unbiased estimator of policy gradient in \eqref{eq: off-policy policy gradient} by using sample average approximation (SAA),  
%by replacing the expectation with the sample mean.
% Similar result for the policy gradient has been seen in \cite{liu2020off} which uses an augmented MDP method to replace the assumption of positive likelihood $\rho_{\pmb\theta_i}>0$. 
%By applying the sample average approximation (SAA), the likelihood ratio based estimator of \eqref{eq: off-policy policy gradient} becomes \vspace{-0.1in}

\begin{equation}
\widehat{\nabla J}^{ILR}_{i,k}
 = \frac{1}{n}
 \sum^{n}_{j=1}\left[\frac{\rho_{\pmb\theta_k}\left(
 \pmb{s}_t^{(i,j)}, \pmb{a}_t^{(i,j)} \right)}{\rho_{\pmb{\theta}_i}\left(\pmb{s}_t^{(i,j)},\pmb{a}_t^{(i,j)} \right)}
      g_k\left(\pmb{a}_t^{(i,j)},\pmb{s}_t^{(i,j)}\right)
  \right] 
  \mbox{~ and ~} 
  \widehat{\nabla J}^{ILR}_{k}
 = \frac{1}{|{U}_k|}\sum_{i\in{U}_k}
\widehat{\nabla J}^{ILR}_{i,k},
  \label{eq: LR-based policy gradient estimator (model-free)}
\end{equation}
where the historical transitions are generated by $%\pmb{s}_i
\pmb{s}_t^{(i,j)}{\sim} d^{\pi_{\pmb\theta_i}}(\pmb{s})$ and
$\pmb{a}_t^{(i,j)} {\sim} \pi_{\pmb\theta_i}(\pmb{a}|\pmb{s}_t^{(i,j)})$ 
for $j=1,2,\ldots,n$.
%$\pmb{a}_i{\sim} \pi_{\pmb\theta_i}(\pmb{a}|\pmb{s}_i)$. 
In this paper, we use ILR to represent (individual) likelihood ratio. The MLR policy gradient estimator is
\begin{equation}
 \widehat{\nabla J}^{MLR}_{k}%=\widehat{\nabla}_{\pmb{\theta}}\mu(\pmb{\pi}_{\pmb{\theta}})
 =\frac{1}{|{U}_k|}\sum_{i\in{U}_k}
 \frac{1}{n}
 \sum^{n}_{j=1}
 %\left[
 f_{k}\left(
\pmb{a}_t^{(i,j)},\pmb{s}_t^{(i,j)}
 \right)
      g_k\left(\pmb{a}_t^{(i,j)},\pmb{s}_t^{(i,j)}\right) ~~~\mbox{with} ~~~
      f_k\left(\pmb{a}_t,\pmb{s}_t
 \right) 
 = \frac{\rho_{\pmb\theta_k}\left(
 \pmb{s}_t, \pmb{a}_t \right)}{ \frac{1}{|{U}_k|}\sum_{i\in{U}_k}\rho_{\pmb{\theta}_i}\left(\pmb{s}_t, \pmb{a}_t\right)}.
 \label{BLR-M sequential gradient estimator}
\end{equation}

The key challenge of utilizing the ILR and MLR policy gradient estimators
is computing %these likelihood ratios of
the stationary distributions $d^\pi(\pmb{s})$ in $\rho_{\pmb{\theta}_k}(\pmb{s},\pmb{a})=\pi_{\pmb{\theta}_k}(\pmb{a}|\pmb{s})d^\pi(\pmb{s})$. This problem is also known as distribution corrections (DICE) in RL. Fortunately, a list of approaches  has been recently proposed to address the challenge; see for example
\cite{NEURIPS2019_cf9a242b,yang2020off}.%zhang2020gradientdice
The off-policy gradient estimator can be simplified % simplifies the likelihood ratio term
by introducing bias. 
\cite{degris2012off} proposed an off-policy (actor-critic) gradient approximate,
\begin{equation} \label{eq: off-policy policy gradient (simplified)}
        \nabla J(\pmb{\theta}_k) \approx \E_{\rho_{\pmb\theta_i}}\left[\frac{\pi_{\pmb\theta_k}}{\pi_{\pmb\theta_i}}\nabla \log \pi_{\pmb\theta_k}(\pmb{a}|\pmb{s})A^{\pi
        _{\pmb{\theta}_k}
        }(\pmb{s},\pmb{a})\right].
\end{equation}
%\cite{degris2012off} argue that this is a good approximation since 
It can preserve the set of local optima to which gradient ascent converges. 
Although biased, this estimator has been widely used in many state-of-the-art off-policy algorithms \citep{schulman2015trust,schulman2017proximal,wang2017sample} due to its simplicity and computational efficiency.

 The ILR and MLR policy gradient estimators for the approximation~\eqref{eq: off-policy policy gradient (simplified)} can be obtained by replacing stationary state-action distribution $\rho_{\pmb\theta}(\pmb{s},\pmb{a})$ with policy $\pi_{\pmb\theta}(\pmb{a}|\pmb{s})$. We conclude this section by
pointing out that we proceed the theoretical analysis with the unbiased off-policy policy gradient \eqref{eq: off-policy policy gradient} and its SAA estimators \eqref{eq: LR-based policy gradient estimator (model-free)}-\eqref{BLR-M sequential gradient estimator} in the following sections while the estimator \eqref{eq: off-policy policy gradient (simplified)} is used in the algorithm and the empirical study.

\subsection{Actor-Critic}
\label{subsec:model free GreenSimulation_PolicyGradient}
The actor-critic is a widely used architecture for policy optimization. It includes two main components: actor and critic. The actor corresponds to an action-selection policy, mapping state to action in a probabilistic manner. The critic
corresponds to a value function, mapping state to expected cumulative future
reward. The actor searches the optimal policy parameters by using stochastic gradient ascent (SGA) while the critic estimates the action-value function $Q^\pi(\pmb{s},\pmb{a})$ by an appropriate policy evaluation algorithm such as temporal-difference learning or Q-learning. 
% Put is simply, the critic addresses a problem of prediction, whereas the actor is concerned with control. These problems are separable, but are solved simultaneously to find an optimal policy.
Usually, the critic $V^\pi(\pmb{s})$ is approximated by a state-value function $V_{\pmb{w}}(\pmb{s})$ specified by parameters $\pmb{w}$ and the actor is represented by a policy function $\pi_{\pmb\theta}(\pmb{a}|\pmb{s})$ specified by $\pmb{\theta}$. Such functional approximation of critic can be used in estimating the state-value function $\hat{V}(\pmb{s})=V_{\pmb{w}}(\pmb{s})$ and thus the TD error \eqref{eq: advantage estimate} becomes $\delta(\pmb{s},\pmb{a},\pmb{s}^\prime)=r(\pmb{s},\pmb{a})+\gamma V_{\pmb{w}}(\pmb{s}^\prime)-V_{\pmb{w}}(\pmb{s})$. 
Following the studies in \cite{bhatnagar2009natural}, a typical actor-critic update can be written as
\begin{align}
    \text{\textbf{TD Error}}:& \quad \delta_k = r_t + \gamma V_{\pmb{w}_k}(\pmb{s}^\prime) - V_{\pmb{w}_k}(\pmb{s}) \label{eq: TD error}\\
     \text{\textbf{Critic}}:&  \quad \pmb{w}_{k+1} = \pmb{w}_{k}  + \eta_{w}\delta_k \nabla_w V_{\pmb{w}_k}(\pmb{s}) \label{eq: critic parameter}\\
     \text{\textbf{Actor}}: & \quad \pmb\theta_{k+1} = \pmb\theta_k +\eta_\theta \nabla J(\pmb\theta)\label{eq: actor parameter}
\end{align}

\noindent where $\eta_w$ and $\eta_\theta$ represent learning rates for critic and actor respectively. \textit{The step in \eqref{eq: TD error} is also referred as to temporal difference learning used to estimate the advantage function.} The policy gradient $\nabla J(\pmb\theta)$ in \eqref{eq: actor parameter} is estimated by 
% either LR estimator \eqref{eq: LR-based policy gradient estimator (model-free)} or 
the MLR policy gradient estimator \eqref{BLR-M sequential gradient estimator}.

% --------------------------

  \section{Variance Reduction Experience Replay for Policy Gradient Estimation}
 \label{sec:algorithms}
 
 In this section, we derive the variance reduced experience replay method %, discuss its computational complexity, 
 and present a general VRER based actor-critic algorithm. 
 We provide a selection criteria in Section~\ref{subsec: selection rule} that can automatically find the most relevant historical transition observations for constructing the reuse set $U_k$ at each $k$-th iteration and improving the policy gradient estimation accuracy.
The dependencies between historical observations collected under selected behavior policies in the previous iterations lead to a general obstacle for most historical sample reusing mechanisms. %One of the reason that the classical experience replay method can succeed is because it 
The MLR, used to leverage the information from previous transition observations, requires sampling distributions to be independent. Thus, in the proposed algorithm, we reduce this interdependence through randomly sampling \citep{mnih2015human}. %However, it becomes tricky when it comes to introducing the MIS technique to policy optimization, because MIS requires sampling distributions to be independent with each other.  To mitigate the dependence problem, 
 Specifically, we separate the optimal policy learning algorithm into two steps. In the online step, we collect new samples by following the target policy specified by $\pmb{\theta}_k$. In the offline step, we select historical samples and train the actor critic model by stochastic gradient ascent. In this way, we can view the offline step as a normal offline optimization problem where samples are assumed to be randomly generated from a set of independent stationary state-action distributions $\rho_i$ with $i\leq k$. Therefore, %in the following theorem, 
for the theoretical study in Section~\ref{subsec: selection rule},
 we assume that the transitions are drawn from a set of independent stationary state-action distributions in the offline optimization step.

 \subsection{Selection Rule for MLR based Policy Gradient Estimator}
 \label{subsec: selection rule}
 We first introduce some properties of MLR based policy gradient estimator. % without providing the proof from multiple importance sampling field. 
 Similar results can be found in \cite{veach1995optimally}, \cite{FengGreenSim2017} and \cite{zheng2021green}.
 
 \begin{lemma} \label{lemma: MLR unbiased}
Conditional on the reuse set ${U}_k$, the MLR policy gradient estimator is unbiased, i.e.,
$$ \E\left[ \left. \widehat{\nabla J}^{MLR}_{k} \right|{U}_k\right]
=\E\left[\left.g_k(\pmb{\tau})
 \right|\pmb{\theta}_k\right]=\nabla J(\pmb\theta_k).
 $$
\end{lemma}

%  Then we present a well known result for MLR based estimator that conditional on $\mathcal{F}_k$: the trace of the covariance of the MLR estimator is less than that of the average ILR estimator in Proposition~\ref{prop: MLR less variance than ILR}.
\begin{proposition}\label{prop: MLR less variance than ILR}
Conditional on % some set of model 
the reuse set
${U}_k$, the total variance of the MLR policy gradient estimator is smaller and equal to that of the average ILR policy gradient estimator,
\begin{equation}\label{eq:inequality of variance policy gradients}
    \mbox{Tr}\left(\mbox{Var}\left[\left.\widehat{\nabla J}^{MLR}_{k}\right|{U}_k\right]\right)\leq\mbox{Tr}\left(\mbox{Var}\left[\left.{\widehat{\nabla J}}^{ILR}_{k}\right|{U}_k\right]\right).
    \nonumber 
\end{equation}
\end{proposition}
\begin{proof}
Similar proofs for multi-dimensional case can be found in \cite[Proposition 1]{zheng2021green}, and see the proof for one-dimensional case in \cite[Theorem~A.2]{martino2015adaptive}.
% and \cite[Proposition~2.5]{FengGreenSim2017}.
\end{proof}

The perspective of multiple importance sampling and variance reduction experience replay is to select and reuse historical transitions generated from those behavioral policies and sampling distributions that are close to the target one. %employ off–policy techniques in order to estimate the performance of target distributions and, consequently, being able to perform multiple gradient steps using the same data, possibly collected with multiple sampling distributions. Specifically, The MLR policy gradient \eqref{eq: LR-based policy gradient estimator (model-free)} %that we use for optimization is a risk–averse constrained estimate, which preventing the replay of historical transitions from behavioral/sampling distributions that are far from the target one. This is justified by the fact that, as we move away from the behavioral distribution, we likely experience larger uncertainty. 
%Bearing the same idea, 
We propose the selection criteria in Theorem~\ref{thm:Online selection rule}, which measures the distance between the behavioral and target distribution based on the variance of policy gradient estimators obtained by using historical samples versus new samples generated in the current $k$-th iteration.

   \begin{theorem}[Selection Rule]
  \label{thm:Online selection rule}
At $k$-th iteration with the target distribution $\rho_k$, the reuse set ${U}_k$ is created to include the stationary distributions, i.e., $\rho_i$ specified by $(\pmb{\theta}_i,\pmb{w}_i)$ with $i\leq k$, whose ILR policy gradient estimator variance is no greater than $c$ times the total variance of the vanilla PG estimator
% at $\pmb{\theta}_k$ 
for some constant $c>1$. Mathematically,
\begin{equation} 
\label{eq: selection rule online}
    \mbox{Tr}\left(\mbox{Var}
    \left[\left. \widehat{\nabla J}^{ILR}_{i,k} \right|M_k\right]\right)\leq c\mbox{Tr}\left(\mbox{Var}
    \left[ \left. \widehat{\nabla J}^{PG}_k \right|M_k\right]\right).
\end{equation}
Then, based on such reuse set ${U}_k$, the total variance of the MLR policy gradient estimator~\eqref{BLR-M sequential gradient estimator} is no greater than ${c}/{|{U}_k|}$ times the total variance of the vanilla PG estimator,
% at $\pmb{\theta}_k$
% Then the MLR policy gradient estimator in (\ref{BLR-M sequential gradient estimator})
% has guaranteed variance reduction, i.e.,
\begin{equation} 
\label{eq: variance reduction (online)}
    \mbox{Tr}\left(\mbox{Var}\left[
    \left. \widehat{\nabla J}^{MLR}_{k} \right|M_k\right]\right)\leq \frac{c}{|{U}_k|}\mbox{Tr}\left( \mbox{Var}\left[
    \left. \widehat{\nabla J}^{PG}_k \right|M_k\right]\right).
\end{equation} 
 \end{theorem}
 
\begin{proof}
We screen all historical models
and select the samples/models that satisfy the rule \eqref{eq: selection rule online}. We denote by ${U}_k$ as the index set of models that are reused or equivalently experience to be replayed. Conditional on all visited models ${M}_k$, the %trajectories $\pmb\tau^{(i,j)}$ 
historical observations, i.e., $\{(\pmb{s},\pmb{a},\pmb{s}^\prime)^{(i,j)}
~~\mbox{with} ~~(\pmb\theta_i,\pmb{w}_i)\in{M}_k 
~~\mbox{and} ~~ j=1,2,\ldots,n\}$, 
are independent. Thus, we have
\begin{eqnarray}
  \mbox{Tr}\left(\mbox{Var}\left[\left.
  \widehat{\nabla J}^{MLR}_{k}\right|M_k\right]\right)
  &\stackrel{(\star)}{\leq}& \mbox{Tr}\left(\mbox{Var}\left[
  \left. \widehat{\nabla J}^{ILR}_{k} \right|M_k\right]\right)
  =\frac{1}{|{U}_k|^2}\sum_{i\in{U}_k}\mbox{Tr}\left(\mbox{Var}\left[ \left. \widehat{\nabla J}^{ILR}_{i,k} \right|M_k\right]\right) \nonumber\\
  &\stackrel{(\star\star)}{\leq}& 
  \frac{c}{|{U}_k|^2}\sum_{i\in{U}_k}\mbox{Tr}\left(\mbox{Var}\left[ \left. \widehat{\nabla J}^{PG}_{k} \right|M_k\right]\right) 
  =
  \frac{c}{|{U}_k|}\mbox{Tr}\left(\mbox{Var}\left[
  \left. \widehat{\nabla J}^{PG}_{k} \right |M_k\right]\right) 
  \label{eq: variance inequality}
  \nonumber 
  \end{eqnarray}
%   \begin{eqnarray}
%   \lefteqn{
%   \mbox{Tr}\left(\mbox{Var}\left[\left.
%   \widehat{\nabla J}^{MLR}_{k}\right|M_k\right]\right)
%   \leq \mbox{Tr}\left(\mbox{Var}\left[
%   \left. \widehat{\nabla J}^{ILR}_{k} \right|M_k\right]\right)} 
%   \label{eq: thm 2: eq 1}\\
%   &=&\frac{1}{|{U}_k|^2}\sum_{i\in{U}_k}\mbox{Tr}\left(\mbox{Var}\left[ \left. \widehat{\nabla J}^{ILR}_{i,k} \right|M_k\right]\right) \nonumber\\
%   &\leq& 
%   \frac{c}{|{U}_k|^2}\sum_{i\in{U}_k}\mbox{Tr}\left(\mbox{Var}\left[ \left. \widehat{\nabla J}^{PG}_{k} \right|M_k\right]\right) \label{eq: thm 2 eq 2 apply selection rule}\\
%   &=&
%   \frac{c}{|{U}_k|}\mbox{Tr}\left(\mbox{Var}\left[
%   \left. \widehat{\nabla J}^{PG}_{k} \right |M_k\right]\right) 
%   \label{eq: variance inequality}
%   \nonumber 
%   \end{eqnarray}
where $(\star)$ follows by applying Proposition~\ref{prop: MLR less variance than ILR} and $(\star\star)$ holds by applying the selection rule \eqref{eq: selection rule online}. %This completes the proof.
\end{proof}

Theorem~\ref{thm:Online selection rule} provides the selection criteria for dynamically and automatically determining the reuse set $U_k$. It shows that the MLR can greatly reduce the policy gradient estimation variance compared to the vanilla policy gradient estimator through reusing the historical transition observations in $U_k$. During the optimal policy search, the number of reuse transitions (or $|{U}_k|$) increases as the iteration $k$ increases. The total variance of the MLR based policy gradient estimator can be significantly reduced.

\subsection{Green Simulation Assisted Policy Gradient Algorithm for Partial Trajectory Reuse} 
\label{subsec: GS-PG algorithm}
   
We summarize the proposed VRER based policy optimization algorithm in an actor-critic framework in Algorithm~\ref{algo: online}.
At each $k$-th iteration, we generate $n$ transitions by running experiments %in the real system 
following the target policy specified by parameters $\pmb{\theta}_k$ and update the observation set $\mathcal{D}_k$ in Step 1. We select the historical samples that satisfy the selection rule~\eqref{eq: selection rule online} and use the associated policies to create the reuse set ${U}_k$ in Step~2.
For computational reasons, we use the policy gradient estimator \eqref{eq: off-policy policy gradient (simplified)} in the algorithm. This approximation simplifies the calculation of the likelihood ratio ${\rho}_k/ \rho_i$ to the likelihood ratio of policies ${\pi_{\pmb\theta_k}}/ {\pi_{\pmb\theta_i}}$ and thus avoid the substantial computation involved in estimating the stationary distribution $d^\pi(\pmb{s})$. 

\begin{algorithm}[ht!]%[H]
  \small
\SetAlgoLined
\textbf{Input}: the selection threshold constant $c$; the maximum number of iterations $K$; the number of iterations in offline optimization $K_{off}$; 
the number of replications per iteration $n$; 
%policy $\pi_{\pmb\theta}(a|s)$, $\forall a \in \mathcal{A}, s \in \mathcal{S}, \pmb\theta \in\mathbb{R}^d$; and initial 
the set of historical trajectories from the real system $\mathcal{D}_0$; the set of policy parameters visited so far ${U}_0$; the set of stored likelihoods $\mathcal{L}_0$.
%and set $\mathcal{F}_1^0=\{\pmb{\theta}_1\}$
%Initialize policy weights $\pmb\theta$. Initialize $\mathcal{E}_{1}$, $\pmb{\Omega}_{1}$ and $\pmb{\Theta}_{1}$ by empty set\; 
\\
\textbf{Initialize} actor parameter $\pmb\theta_1$ and critic parameter $\pmb{w}_1$. Store them in $M_1=M_0\cup\left\{(\pmb\theta_1, \pmb{w}_1\right)\}$\;
 \For{$k=1,2,\ldots, K$ %(at each new real-world data collection)
 }{
%  1. Generate posterior samples $\pmb{\omega}_{k} \sim p(\pmb{\omega}|\mathcal{D})$ and build the transition model with new parameter $\pmb{\omega}_{k}$, i.e., $p(\pmb{s}_{t+1}|\pmb{s}_t,\pmb{a}_t,\pmb{\omega}_{k})$ for $t=1,2,\ldots,H-1$ \;

 1. Collect a set of transitions $\mathcal{T}_k=\left\{(\pmb{s}_t,\pmb{a}_t, \pmb{s}_{t+1},r_t)\right\}_{t=1}^{n}$ from real system by running policy $\pmb{\pi}_{\pmb\theta_k}(\pmb{a}_t|\pmb{s}_t)$; Update the sets $\mathcal{D}_k \leftarrow \mathcal{D}_{k-1} \cup\mathcal{T}_k$\;
%  $\mathcal{D}_k$\;
%  by using the exploration policy $g_{\pmb\beta_k}(\pmb{a}_t|\pmb{s}_t^{(k,j)}, \pmb{\omega}_{k})$

 %\For{$j=1,2,\ldots,n$}{
%  (a) Generate new trajectories
%  $\mathcal{T}_k \equiv
% \left\{ \pmb\tau^{(k,j)} \sim D\left(\pmb\tau; {\pmb{\theta}}_{k},\pmb\omega^c\right)
% : j=1,2,\ldots,n
% \right\}$; 
 
%  (b) Store new trajectories $\mathcal{D}_k \leftarrow \mathcal{D}_{k-1} \cup\mathcal{T}_k$
 %\{\pmb{\tau}^{(k,j)}:j=1,2,\ldots,n\}$\;

 2. Initialize ${U}_k=\emptyset$, screen all historical transitions and policies in ${U}_k$, and construct the reuse set ${U}_k$\;
 \For{$(\pmb\theta_i,\pmb{w}_i) \in {M}_k$}{
 (a) Compute and store the new likelihoods:
 \\ \quad \ \ $\mathcal{L}_k\leftarrow\mathcal{L}_{k-1}\cup\pi_{\pmb\theta_k}(\mathcal{D}_k)\cup\pi_{\pmb\theta_{[1:k]}}(\mathcal{T}_k)$
%  ; see the definition in \eqref{eq: likelihood storage (online)}.
 %involving either new policy parameters or new generated trajectories:
 \\

 (b) 
%  Compute $\mbox{Tr}\left(\mbox{Var}\left[\left.\widehat{\nabla J}^{ILR}_{i,k}\right|M_k\right]\right)$ and $\mbox{Tr}\left(\mbox{Var}\left[\left.\widehat{\nabla J}^{PG}_{k}\right|M_k\right]\right)$.\\
%  by using the estimator \eqref{eq.PG-VarEst}. 
%  Then construct ${U}_k$ by applying (\ref{eq: selection rule online});
 \textbf{if} \ $\mbox{Tr}\left(\mbox{Var}\left[\left.\widehat{\nabla J}^{ILR}_{i,k}\right|M_k\right]\right) \leq c\mbox{Tr}\left(\mbox{Var}\left[\left.\widehat{\nabla J}^{PG}_{k}\right|M_k\right]\right)$ \ \textbf{then} \ ${U}_k \leftarrow {U}_k\cup \{ i\}.$
%  \If{$\mbox{Tr}\left(\mbox{Var}\left[\left.\widehat{\nabla J}^{ILR}_{i,k}\right|M_k\right]\right) \leq c\mbox{Tr}\left(\mbox{Var}\left[\left.\widehat{\nabla J}^{PG}_{k}\right|M_k\right]\right)$}{ ${U}_k \leftarrow {U}_k\cup \{ i\}.$ }
%  by \eqref{eq: vanilla policy gradient estimator}.
 }
 3. Reuse the historical samples associated with ${U}_k$ and stored likelihoods $\mathcal{L}_k$ to update actor and critic:
 
(a) Let $\pmb\theta_k^0=\pmb\theta_k$ and $\pmb{w}_k^0=\pmb{w}_k$\;

\For{$h=0, 1, ..., K_{\text{off}}$}{
% calculate the policy gradient $\widehat{\nabla\mu}^{MLR}_{k}$ by applying \eqref{BLR-M sequential gradient estimator}\;
(b) \textbf{TD Error}: $\delta_k^h = r_t + \gamma V_{\pmb{w}_k^h}(\pmb{s}^\prime) - Q_{\pmb{w}_k^h}(\pmb{s})$\;
(c) \textbf{Actor Update}:
 $\pmb{\theta}_{k}^{h+1} \leftarrow \pmb{\theta}_k^h+ \eta_k\widehat{\nabla J}^{MLR}_{k}$\;
(d) \textbf{Critic Update}:
$\pmb{w}_{k}^{h+1} = \pmb{w}_{k}^h  + \eta_k\delta_k \nabla_w V_{\pmb{w}_k^h}(\pmb{s})$\;
 }
%  Set $\mathcal{F}_{k+1}^0=\mathcal{F}_{k}\cup \{\pmb\theta_{k+1}\}$.
 
 %\quad \ \ $\pmb{\theta}_{k+1} \leftarrow \pmb{\theta}_k+ \eta_k\widehat{\nabla\mu}^{MLR}_{k}$.
 4. Update the actor and critic: $\pmb\theta_{k+1} = \pmb\theta_{k}^{K_{off}}$ and $\pmb{w}_{k+1} = \pmb{w}_{k}^{K_{off}}$\;
   5. Store them to the set ${M}_{k+1}={M}_k\cup\left\{(\pmb\theta_{k+1},\pmb{w}_{k+1})\right\}$;\
  
 }
\caption{Actor Critic Method with Variance Reduced Experience Replay.}\label{algo: online}
\end{algorithm}

The likelihoods are stored and reused in Step 2(a) to reduce the computation cost; see more details in \cite{zheng2021green}.
Specifically, as the iteration $k$ increases, the number of historical transitions increases. 
It can be computationally expensive to repeatedly calculate all likelihood ratios required for historical observation selection and policy gradient estimation. Thus, we save and reuse the previous calculated likelihoods. Let $\mathcal{T}_k=\{(\pmb{s}_{t},\pmb{a}_{t},\pmb{s}_{t+1},r_t)\}_{t=1}^{n}$ represent the
set of new transitions generated in the $k$-th iteration. 
Update the set of all transition observations, i.e., $\mathcal{D}_k \leftarrow \mathcal{D}_{k-1} \cup\mathcal{T}_k$. 
Then, the likelihoods of $\mathcal{T}_k$ under any previous visited policy $\pmb{\theta}_i$ are $\pi_{\pmb\theta_i}(\mathcal{T}_k)=\{\pi_{\pmb\theta_i}(\pmb{a}_t|\pmb{s}_t): (\pmb{s}_{t},\pmb{a}_{t},\pmb{s}_{t+1},r_t)\in\mathcal{T}_k\}$. Let $\pi_{\pmb{\theta}_{[1:k]}}(\mathcal{T}_k)=\{\pi_{\pmb\theta_i}(\mathcal{T}_k): i=1,\ldots,k\}$.
%Thus, the likelihoods of transition $(\pmb{s}_{t},\pmb{a}_{t},\pmb{s}_{t+1},r_t)$ under policy functions $\{\pmb{\theta}_1,\ldots,\pmb{\theta}_k\}$ are represented as $\pi_{\pmb{\theta}_{[1:k]}}(\pmb{a}_t|\pmb{s}_t)=\{\pi_{\pmb\theta_i}(\pmb{a}_t|\pmb{s}_t): i=1,\ldots,k\}$. 
Therefore, at the $k$-th iteration, all newly generated likelihoods are the joint set of the values of historical samples at new policy $\pmb{\theta}_k$ and new samples at historical policies $\pmb{\theta}_i$, i.e., $\pi_{\pmb\theta_k}(\mathcal{D}_k)\cup \pi_{\pmb{\theta}_{[1:k]}}(\mathcal{T}_k)$. Then we %reuse the selected trajectories in ${U}_k$ and 
get  
%calculated likelihoods to 
the MLR policy gradient estimate by applying (\ref{BLR-M sequential gradient estimator})
and train the actor and critic by following \eqref{eq: TD error}-\eqref{eq: actor parameter} in Step 3. In this offline step, we use the stochastic gradient ascent \citep{li2014efficient} to iteratively optimize the performance objective with the historical observations in ${U}_k$. At the end, we update actor and critic with latest parameters $\pmb\theta_{k}^{K_{off}}$ and $\pmb{w}_{k}^{K_{off}}$ from the offline optimization step and store them into memory buffer of models $M_{k+1}$. After that, we repeat the procedure until reaching to the budget limit specified by $K$ iterations.

%Section \ref{subsec: selection rule properties}.

We conclude this section by pointing out that the choice of optimizers for actor/critic training, the hyperparamter $K_{off}$ and the size of minibatch in SGA are task-specific. In some actor-critic algorithms, the termination of offline optimization or the number of iterations $K_{off}$ are often not fixed. For example, PPO uses early stopping method to determine $K_{off}$: terminate the offline training if the KL divergence between behavioral and target policies is smaller than some threshold. The interested reader is referred to the literature of stochastic gradient methods \citep{Goodfellow-et-al-2016} for details of hyperparamter tuning.

\section{Empirical Study}
\label{sec: empirical study}
In this section, we present the empirical study assessing the performance VRER in combination with actor critic algorithm \citep{bhatnagar2009natural} and PPO algorithm \citep{schulman2017proximal}. 
We study the optimization convergence behavior by using control tasks in Section \ref{subsec: convergence}, present the sensitivity analysis on the reuse set selection threshold constant $c$ in Section~\ref{subsec: sensitivity analysis}, and investigate the effects of employing VRER to the gradient variance reduction in Section \ref{subsec: variance reduction}. 
For the implementation, we use two open-sourced libraries, Keras
and TensorFlow 
% \citep{tensorflow2015-whitepaper} 
for policy modeling and automatic differentiation. In addition, we use OpenAI gym \citep{1606.01540} to provide the simulation environment of Cartpole and Acrobot problems. 

We adopt the yeast cell fermentation simulator from \cite{zheng2021policy}. 
To provide the prediction on the process dynamics, we add 4 additional state variables, including time $t$, the growth rate $\dot{X}_f$, the production rate of citrate acid $\dot{C}$, and the consumption rate of substrate $\dot{S}$. Therefore, the state vector is $\pmb{s}=(X_{f},C,S,N,V,t,\dot{X}_f, \dot{C}, \dot{S})$, where $X_f$ represents lipid-free cell mass; $C$ measures the citrate concentration, i.e., the actual ``product'' to be harvested at the end of the fermentation process, generated by the cells' metabolism; $S$ and $N$ are the amounts of substrate (a type of oil) and nitrogen, used for cell growth and production; and $V$ is the working volume of the entire batch. %By using the fermentation simulator \citep{zheng2021policy}, 
In this paper, we consider a setpoint control problem that aims to maintain the substrate concentration $S_t$ around a fixed value. The reward function is defined as $r(\pmb{s}_t)=
-(S_t-S_t^0)^2
%-(\pmb{s}_t-\pmb{s}_t^0)^2
$, where $S_t^0=20$ g/L is the setpoint of substrate concentration. We consider the feed rate as the action representing the amount of substrate added in each time interval.

The actor and critic models for Actor-Critic and PPO are adopted from the Keras implementation \citep{chollet2015keras}. The Actor-Critic model is composed of a shared initial layer with 128 neurons and separate outputs for the actor and critic. The PPO algorithm has separate actor and critic neural network models, both of which have two layers with 64 neurons. For the problems with discrete action, we use softmax activation function on top of the actor network, which calculates the probability of optimal actions. For the fermentation problem with a continuous action (feeding rate of substrate), we use the Gaussian policy for actor model \citep{sutton2018reinforcement}. 
%For continuous control, Gaussian policy is a commonly used approach \citep{sutton2018reinforcement} by assuming the optimal action follows a Gaussian distribution. 
As the feed rate is strictly regulated and it should stay within a regulation required acceptance range, we truncate the action sampled from Gassuain policy. At each $k$-th iteration, based on the results obtained from 30 macro-replications, we represent the estimation uncertainty of outputs (i.e., the expected discounted rewards and the total variance of policy gradient) by using the 95\% confidence bands based on asymptotic normality assumption. 
% Please check the Github repository for more implementation details \url{https://github.com/zhenghuazx/vrer_policy_optimization}.
% \footnote{Github repository: \url{https://github.com/zhenghuazx/vrer_policy_optimization}}.

\subsection{Comparison of Algorithm Performance with and without Proposed VRER}
\label{subsec: convergence}

In this section, we compare the optimal policy learning performance of VRER using Actor-Critic and PPO algorithms on some classical continuous control benchmarks. We set the same initial learning rate for both actor and critic in Actor-Critic algorithm (i.e., Cartpole: $0.005$; Acrobot:  $0.001$; and Fermentation: 0.001). For PPO, the learning rates of actor and critic were set to be 0.001 and 0.005 respectively for all three examples. The selection threshold constant was set to be $c=1.5$ for all experiment in this section.

\begin{figure}[!thb]
     \centering
     \subfloat[CartPole]{%
         \centering
         \includegraphics[width=0.31\textwidth]{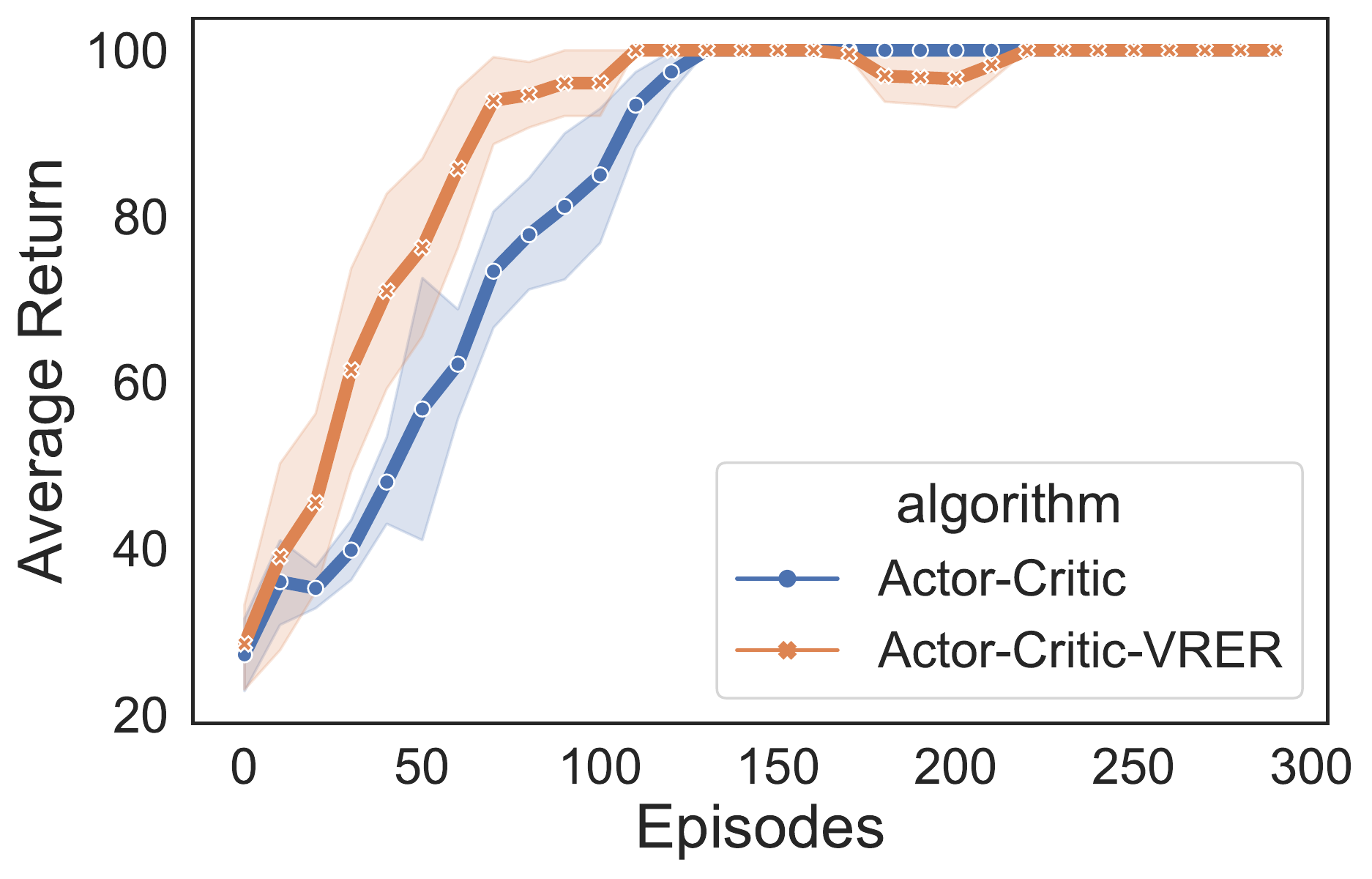}
         \label{fig: actor-critic convergence (CartPole)}}
     \subfloat[Acrobot]{
         \centering
         \includegraphics[width=0.32\textwidth]{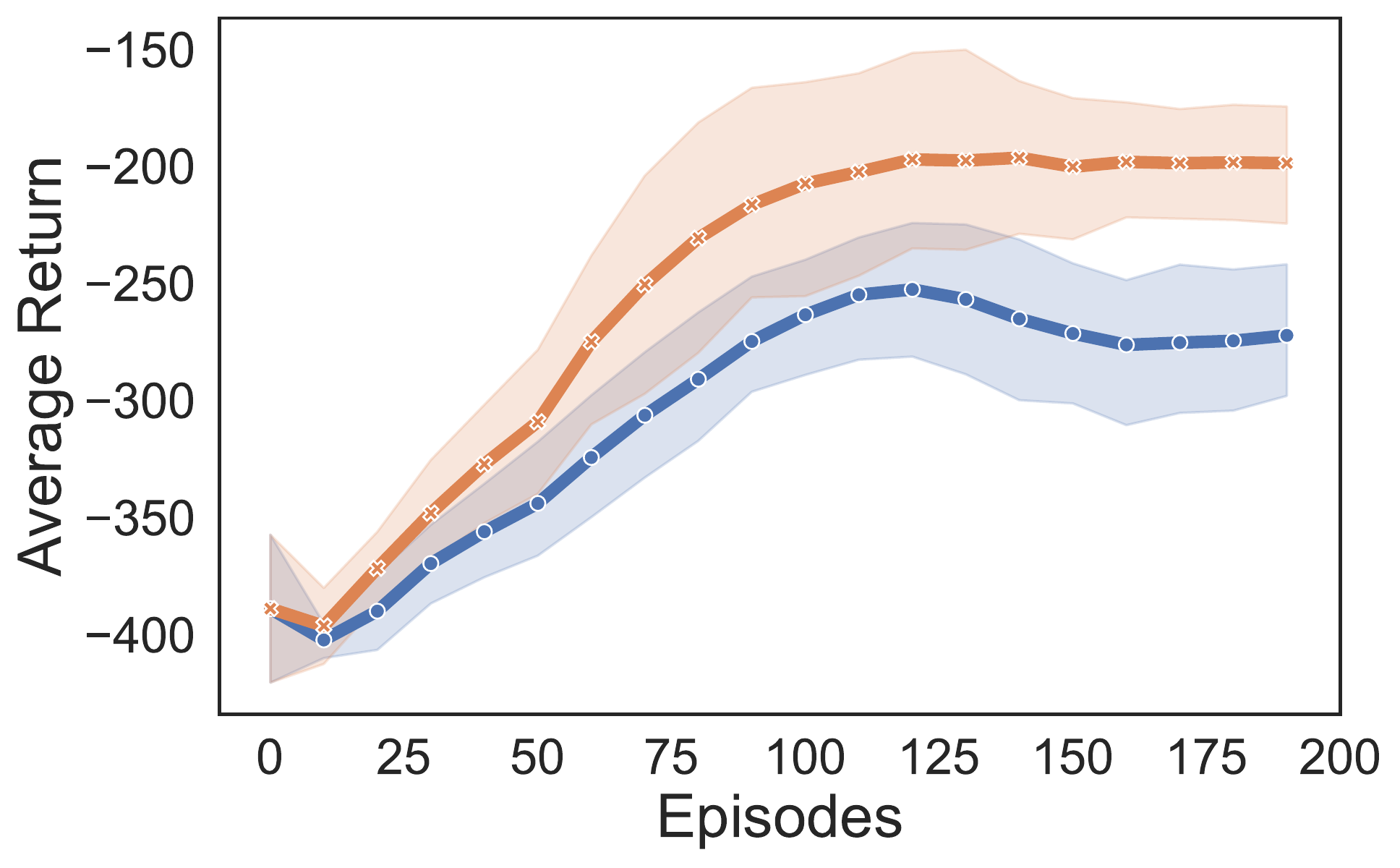}
         \label{fig: actor-critic convergence (Acrobot)}
     }
     \subfloat[Fermentation]{
         \centering
         \includegraphics[width=0.315\textwidth]{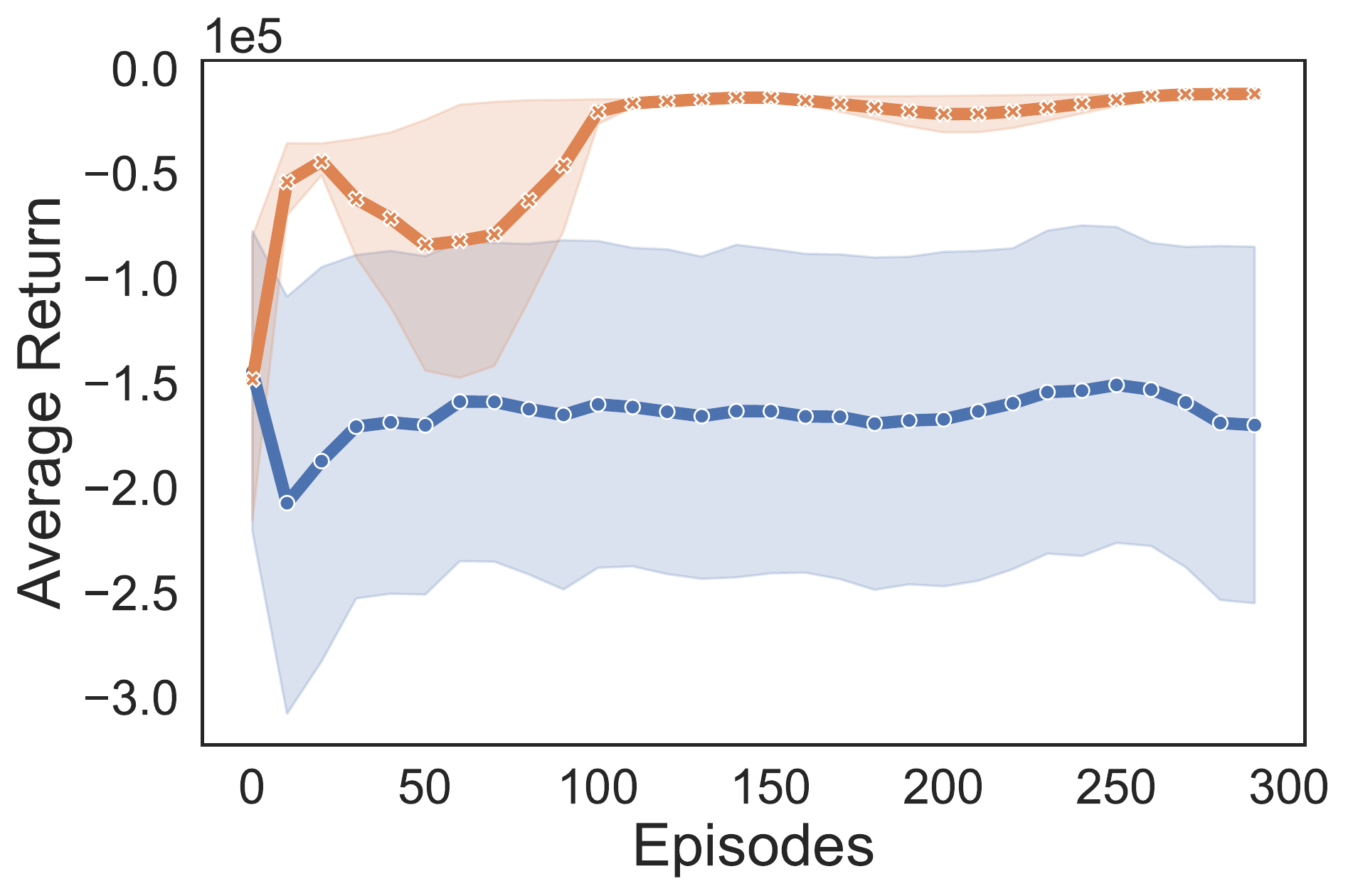}
         \label{fig: actor-critic convergence (Fermentation)}
     }
     \medskip
     \vspace{-0.1in}
     \caption{Convergence results for the Actor-Critic algorithm with and without using the proposed VRER. %Result are smoothed by the moving average with a smoothing window width of 100.
     } 
     \label{fig: actor-critic convergence}
     \vspace{-0.1in}
\end{figure}

We plot the mean performance curves and 95\% confidence intervals of the actor-critic with and without VRER in Figure~\ref{fig: actor-critic convergence}. For the Cartpole problem, although the Actor-Critic algorithms converge to the optimum with and without using VRER, the Actor-Critic-VRER shows significantly faster convergence than the Actor-Critic. This indicates that the use of VRER gives significant performance improvement.
Similar performance improvement can be also seen in Acrobot example (Figure~\ref{fig: actor-critic convergence (Acrobot)}), where Actor-Critic-ARER shows not only the convergence to the optimum but
also faster convergence. For the fermentation problem, Actor-Critic-VRER shows performance improvement while Actor-Critic method even fails to converge. 
%we have a less consistent behavior: Actor-Critic-VRER only shows insignificantly better performance than Actor-Critic method since the mean performance curve of Actor-Critic stays within the confidence interval of Actor-Critic-VRER. Similar performance issue of step-based approaches in the fermentation example was previously reported by \cite{zheng2021policy}: the deep deterministic policy gradient (DDPG) algorithm achieve only 55 average return after 400 iterations.

We plot the mean performance curves and 95\% confidence intervals of PPO in Figure~\ref{fig: ppo convergence}. In Cartpole, the average return of PPO-VRER converges about 25 iterations earlier than PPO. In Acrobot and Fermentation problems, PPO-VRER shows better performance with higher average return and lower variance.
% The performance improvement of VRER is also significant for PPO.
% This can be explained by the fact that the PPO algorithm chips the likelihood ratio to avoid large uncertainty and penalize too large changes to the policy. Such likelihood ratio chipping basically serves an alternative for mixture likelihood ratio in PPO.

% while the mean performance curve of Actor-Critic-VRER is consistently higher than Actor-Critic.

\begin{figure}[!thb]
%  \vspace{-0.2in}
     \centering
     \subfloat[CartPole]{%
         \centering
         \includegraphics[width=0.309\textwidth]{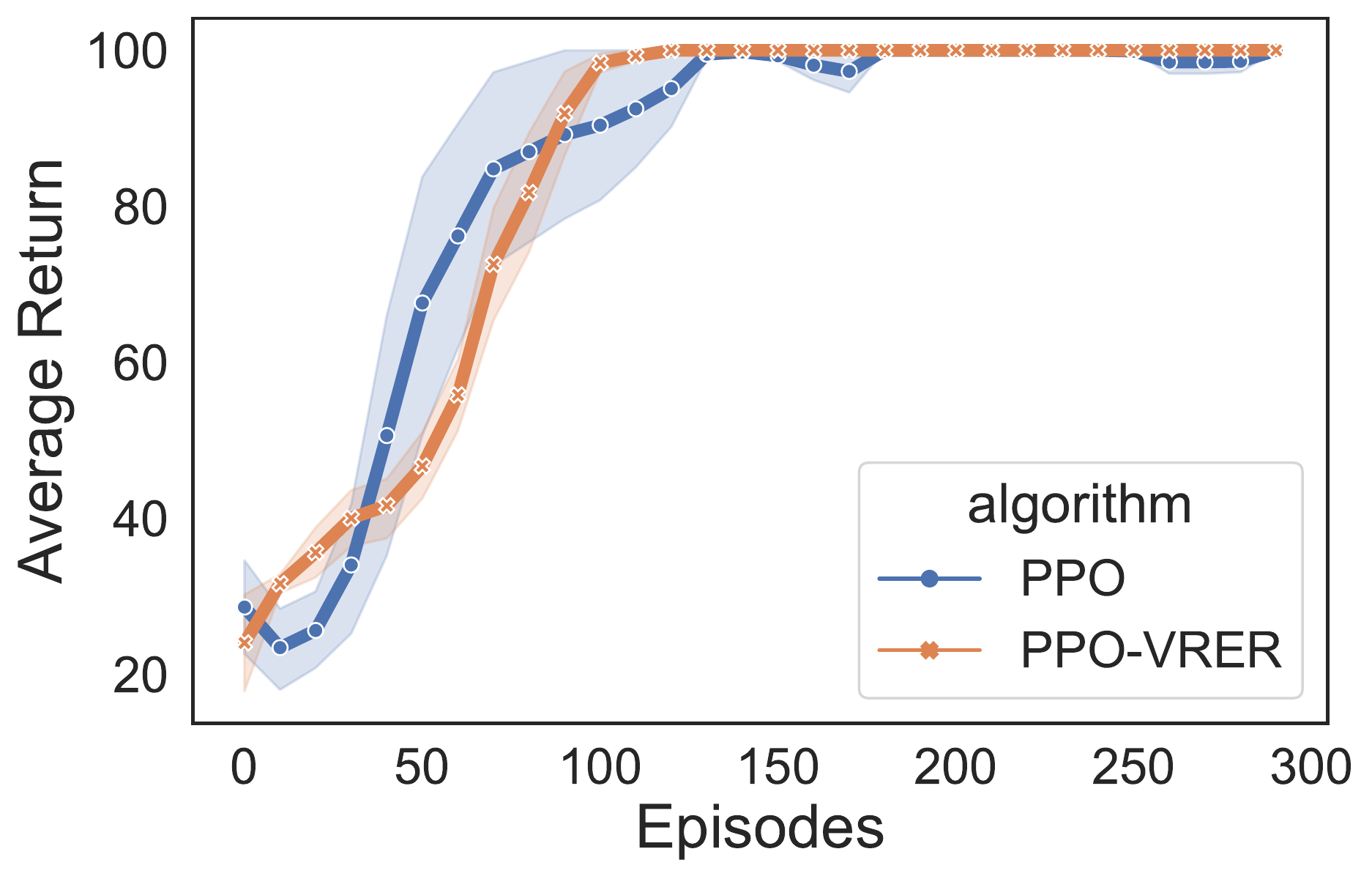}
         \label{fig: PPO convergence (CartPole)}}
     \subfloat[Acrobot]{
         \centering
         \includegraphics[width=0.32\textwidth]{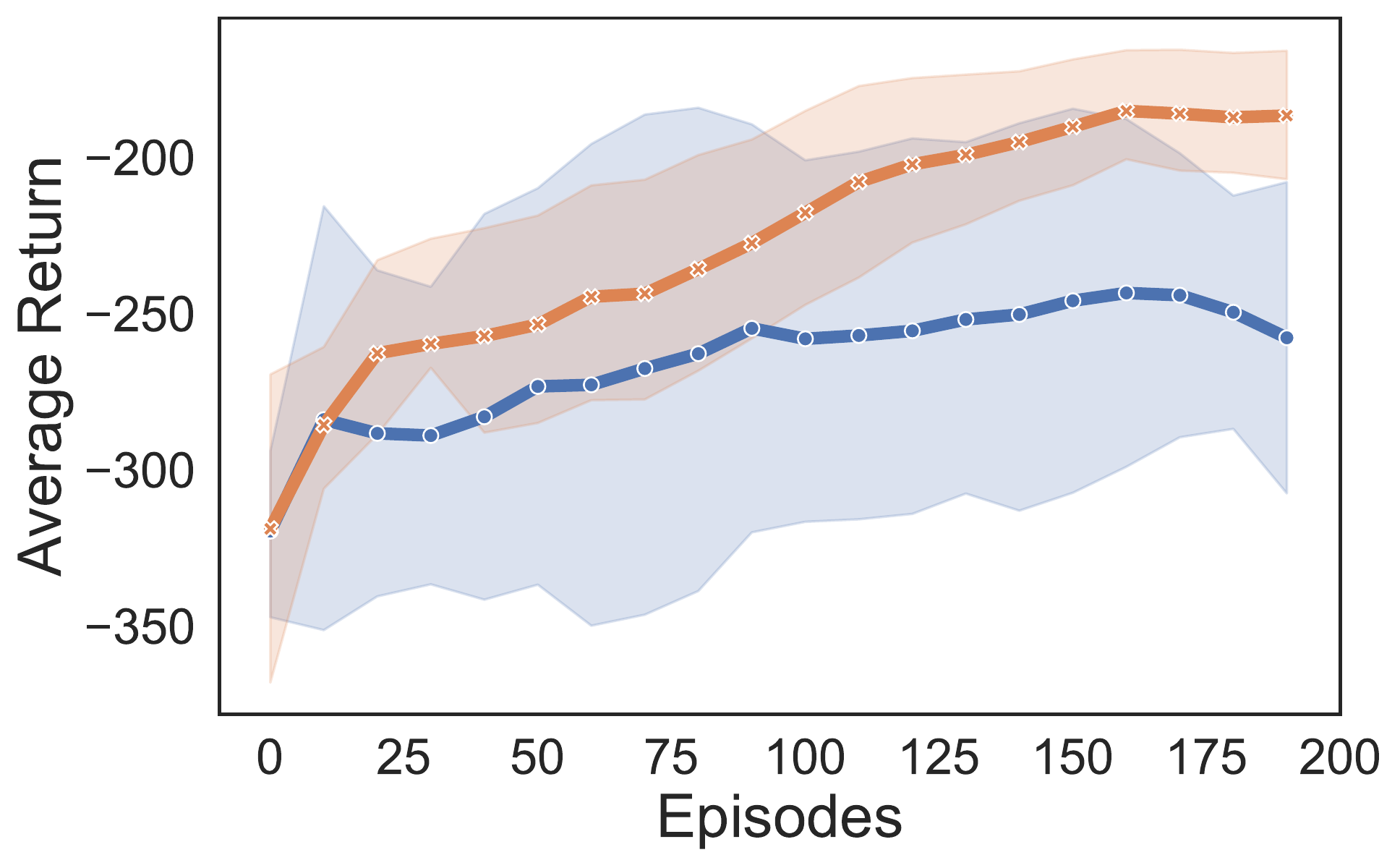}
         \label{fig: PPO convergence (Acrobot)}
     }
     \subfloat[Fermentation]{
         \centering
         \includegraphics[width=0.315\textwidth]{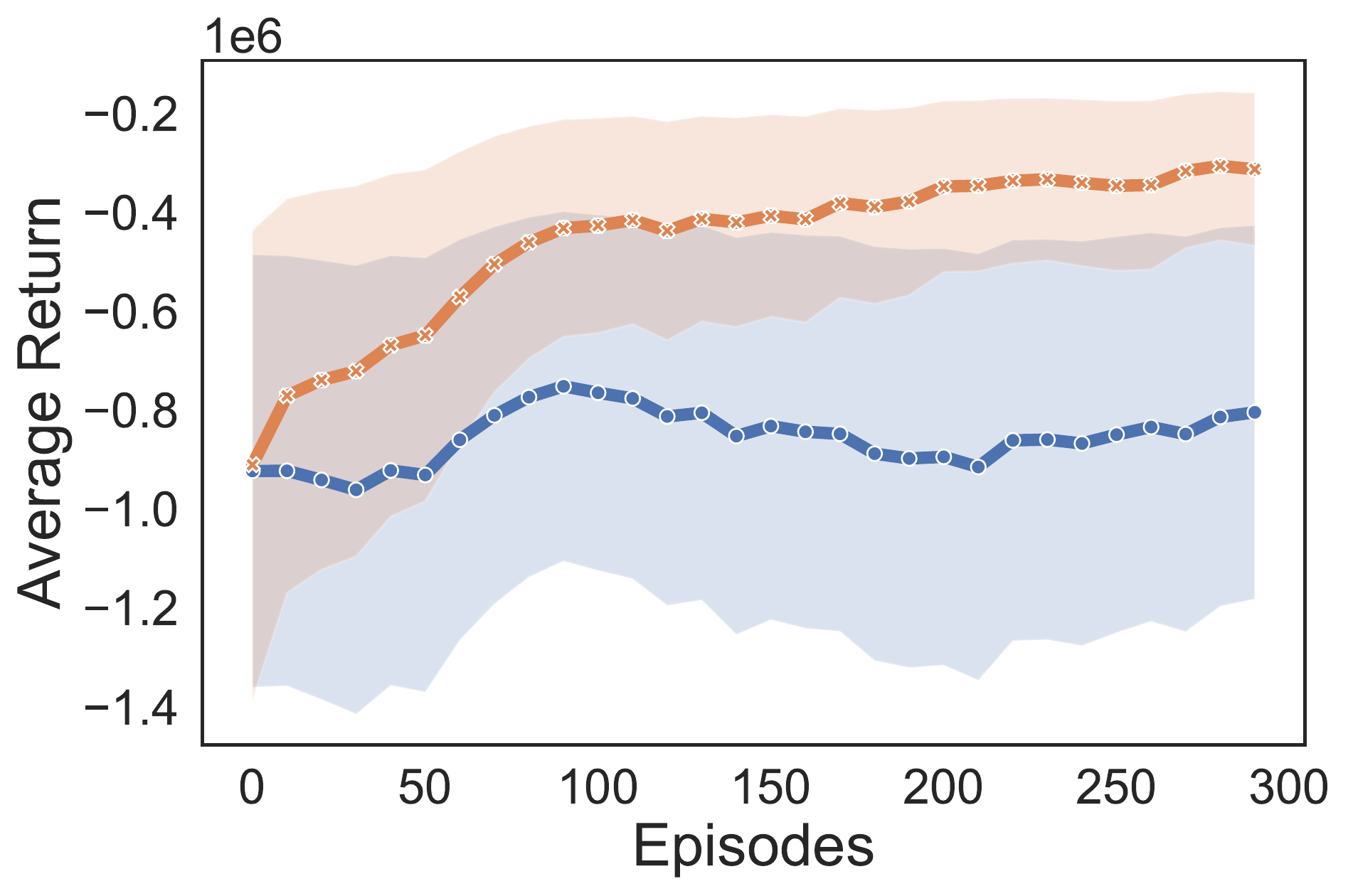}
         \label{fig: PPO convergence (Fermentation)}
     }
     \medskip
     \vspace{-0.1in}
     \caption{Convergence results for the PPO algorithm with and without using the proposed VRER. %Result are smoothed by the moving average with a smoothing window width of 100.
     } 
     \label{fig: ppo convergence}
     \vspace{-0.1in}
\end{figure}

\subsection{Sensitivity Analysis on the Selection of Reuse Threshold $c$}
\label{subsec: sensitivity analysis}

We study the sensitivity of convergence performance of actor critic and PPO algorithms with VRER to the choice of constant $c$ used in the selection criteria \eqref{eq: selection rule online}. Figure \ref{fig: sensitivity} shows the convergence behaviors when we solve the Cartpole problem with different values of $c$. All the performance curves stay close, which indicates that the convergence of the proposed VRER based policy optimization approach is robust to the choice of $c$. 

\begin{figure}[!thb]
     \centering
     \subfloat[Actor Critic]{%
         \centering
         \includegraphics[width=0.4\textwidth]{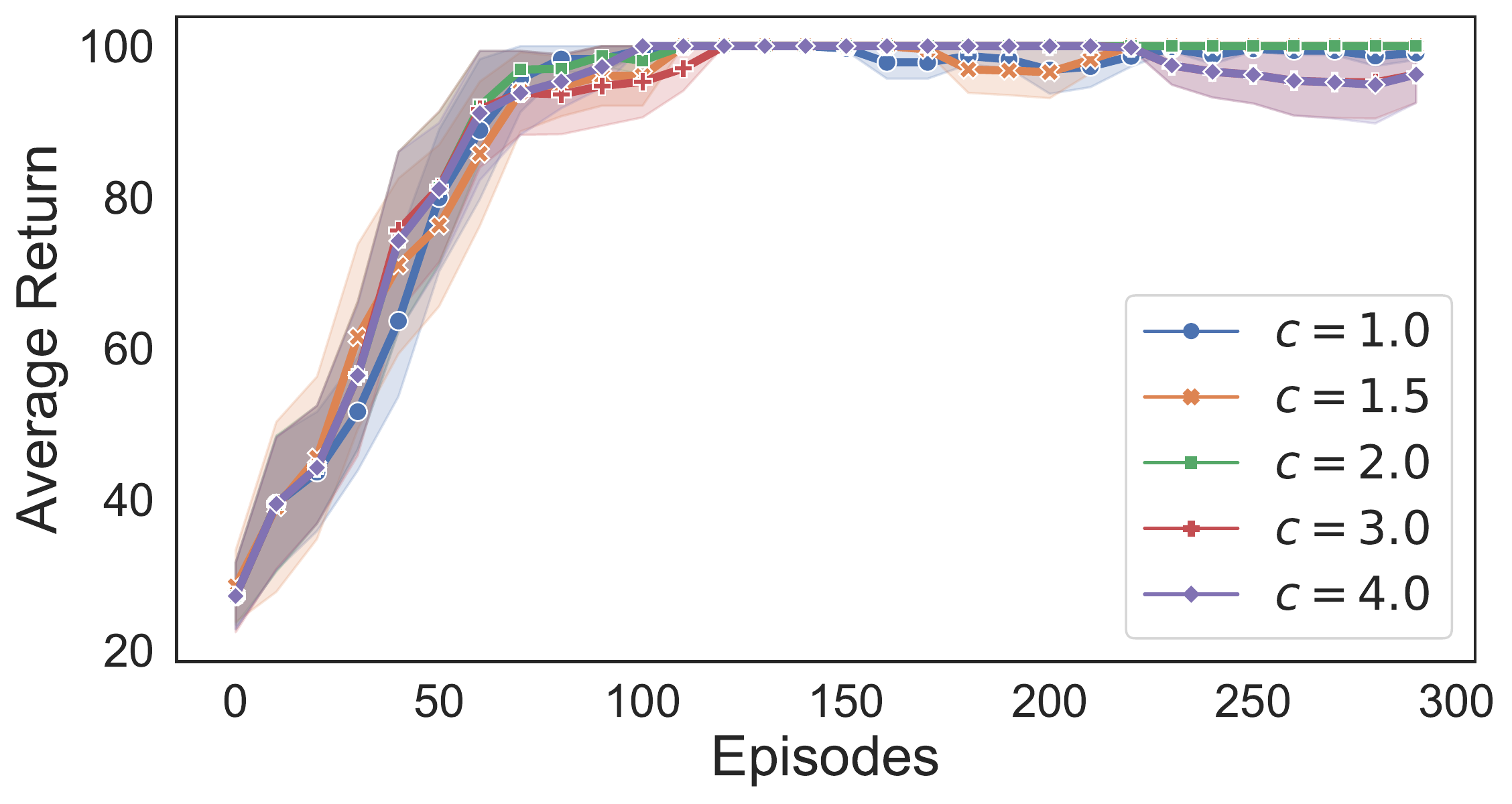}
         \label{fig: actor-critic sensitivity (CartPole)}}
     \subfloat[PPO]{
         \centering
         \includegraphics[width=0.4\textwidth]{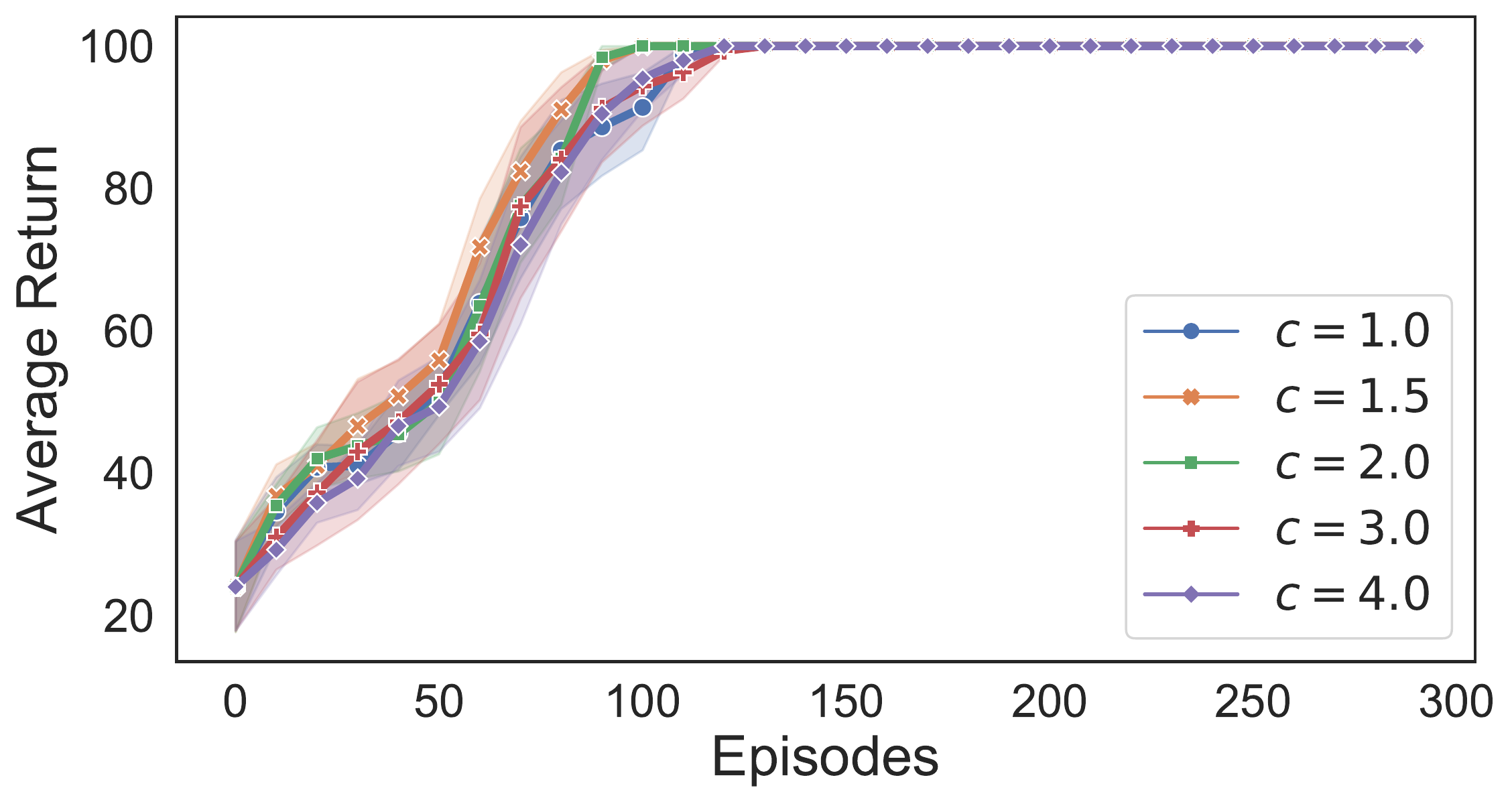}
         \label{fig: actor-critic sensitivity (Fermentation)}
     }
     \medskip
     \caption{Sensitivity analysis of the reuse selection threshold constant $c$ in Cartpole example. %Result are smoothed by the moving average with a smoothing window width of 100.
     } 
     \label{fig: sensitivity}
     \vspace{-0.2in}
\end{figure}

\subsection{Variance Reduction}
\label{subsec: variance reduction}
In this section, we present empirical results to assess the performance of the proposed VRER in terms of reducing the policy gradient estimation variance. %validate the properties described in Section~\ref{sec:algorithms}. 
We first test the proposed VRER method in conjunction with actor-critic algorithm (Actor-Critic-VRER) in three distinct control examples (Figure~\ref{fig: actor-critic gradient norm}). The original actor-critic method is an on-policy reinforcement learning algorithm that thus suffers from the high variability of gradient estimators. By selectively reusing historical transition observations through the VRER-based  
MLR approach, the Actor-Critic algorithm shows a significant reduction in the estimation variance of policy gradient in all three examples, compared to the original policy gradient without any experience replay. 

Similar results are also observed for the PPO algorithm. The PPO, instead of using multiple importance sampling, clips/truncates likelihood ratio for policy regularization and therefore eliminates the inflated gradient variance caused by extreme samples \citep{schulman2017proximal}. The chipping technique, as an alternative to MLR method, provides a simple and computation efficient method to regularize the policy gradient and adjust distributional difference. However it introduces extra bias and thus may cause the algorithm stuck at suboptimum. The results in Figure~\ref{fig: ppo gradient norm} show that the use of VRER can still reduce the estimation variance of PPO policy gradient estimator even if MLR is replaced by chipped likelihood ratio.

\begin{figure}[!thb]
     \centering
     \subfloat[CartPole]{%
         \centering
         \includegraphics[width=0.315\textwidth]{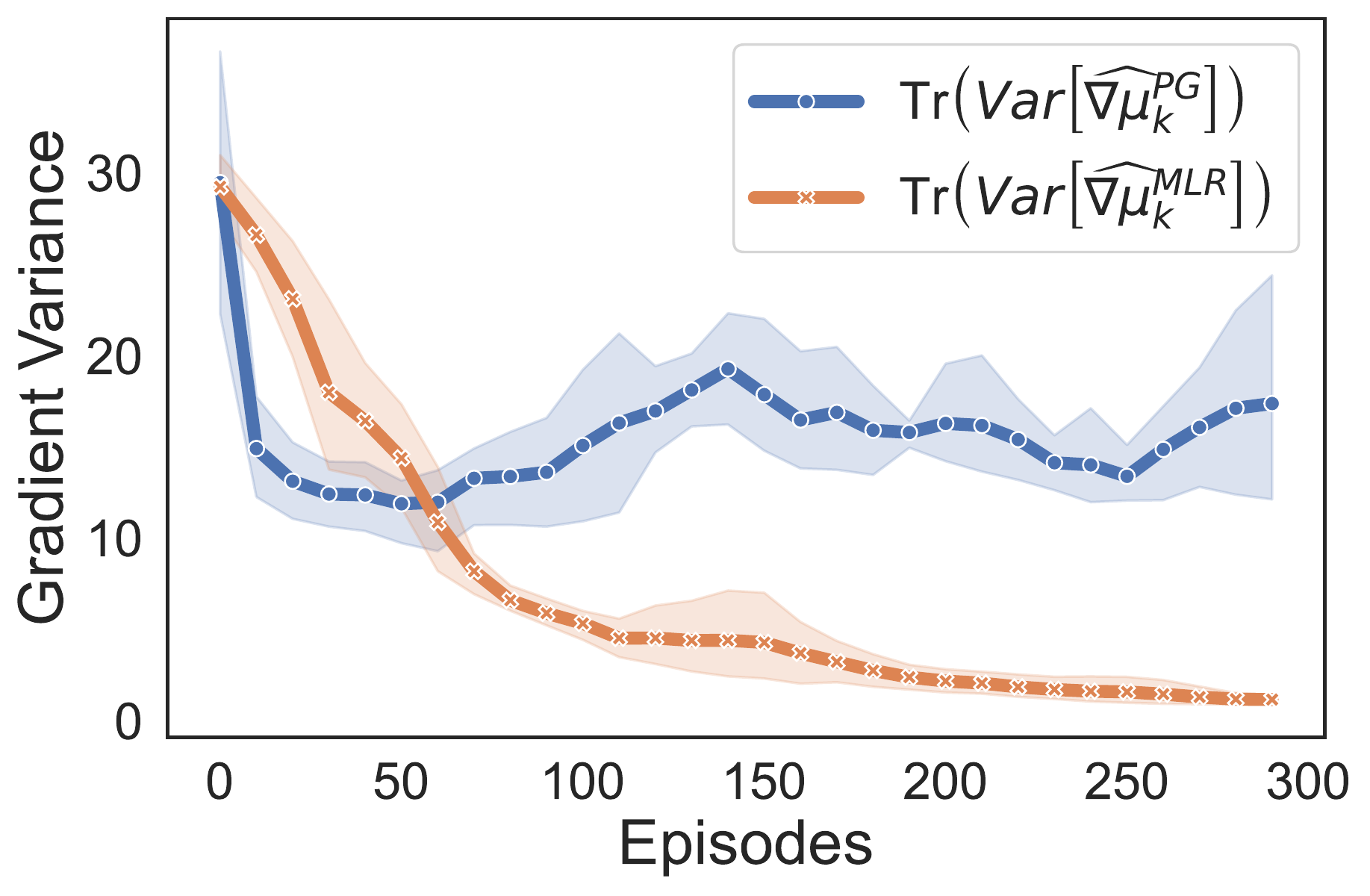}
         \label{fig: actor-critic gradient norm (CartPole)}}\hspace{-0.08in}%
     \subfloat[Acrobot]{
         \centering
         \includegraphics[width=0.33\textwidth]{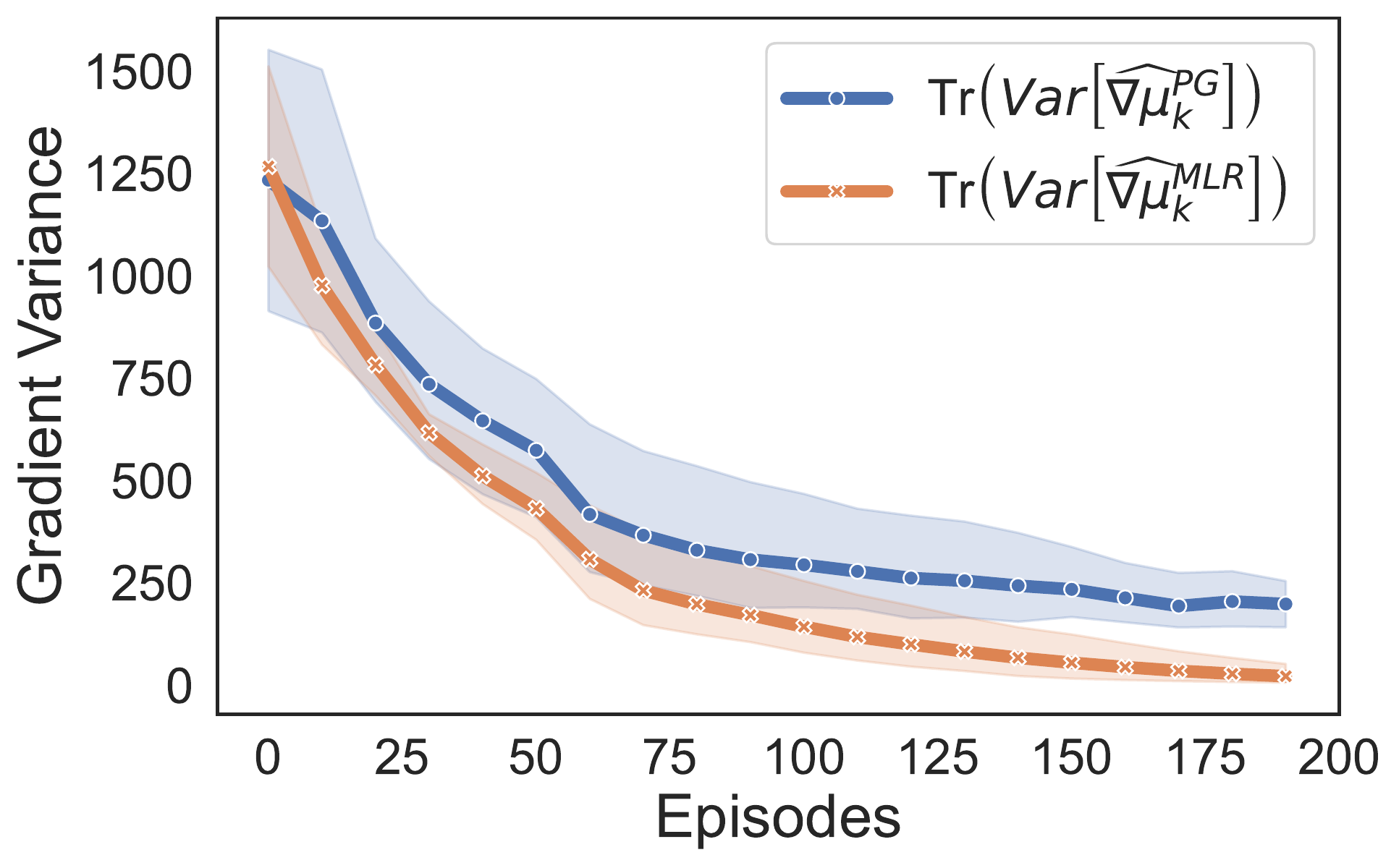}
         \label{fig: actor-critic gradient norm (Acrobot)}
     }\hspace{-0.12in}%
     \subfloat[Fermentation]{
         \centering
         \includegraphics[width=0.33\textwidth]{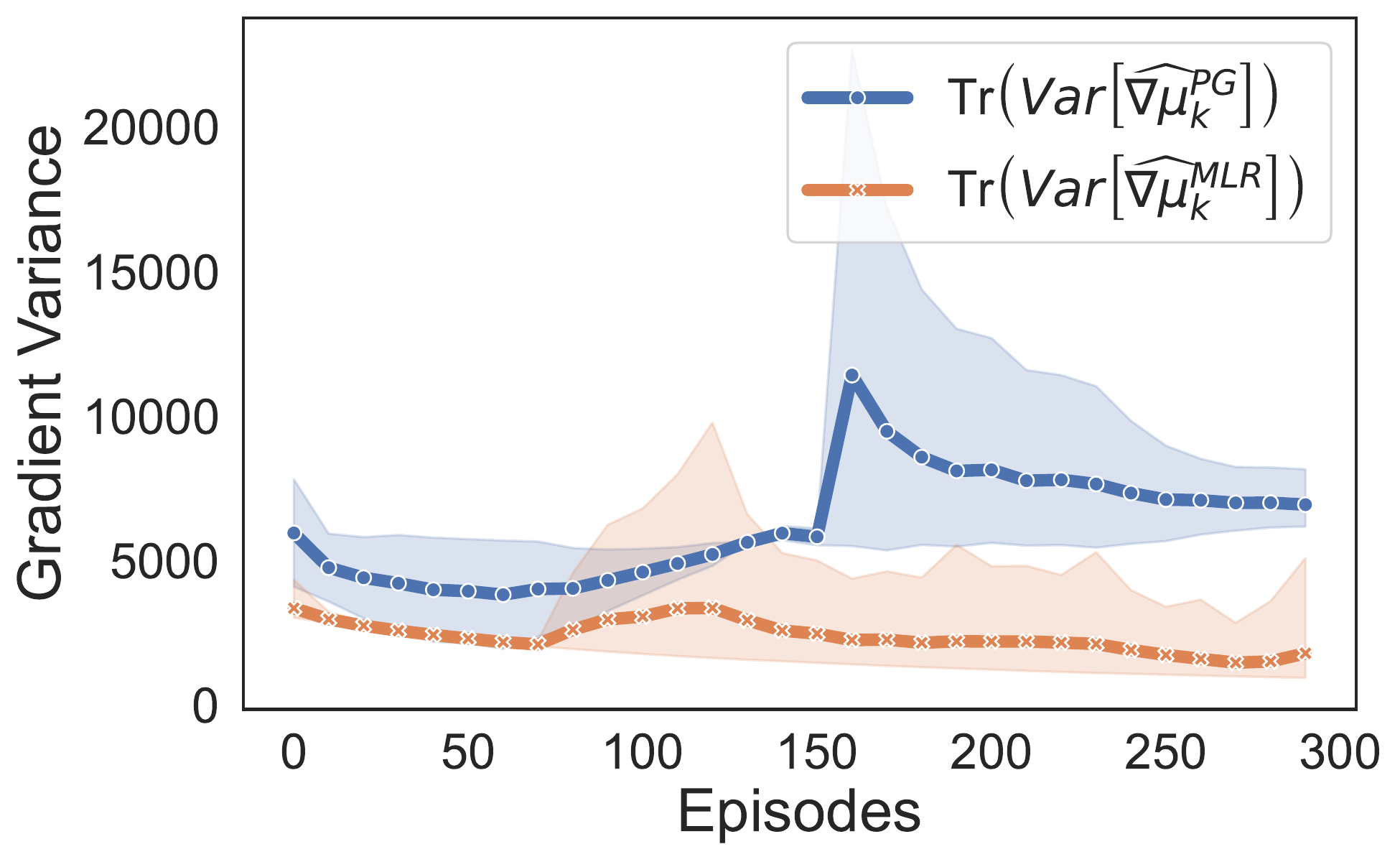}
         \label{fig: actor-critic gradient norm (Fermentation)}
     }
     \medskip
     \vspace{-0.1in}
     \caption{Policy gradient estimation variance results of  Actor-Critic algorithm with and without VRER. %Result are smoothed by the moving average with a smoothing window width of 100.
     } 
     \label{fig: actor-critic gradient norm}
     \vspace{-0.2in}
\end{figure}

\begin{figure}[!thb]
     \centering
     \subfloat[CartPole]{%
         \centering
         \includegraphics[width=0.315\textwidth]{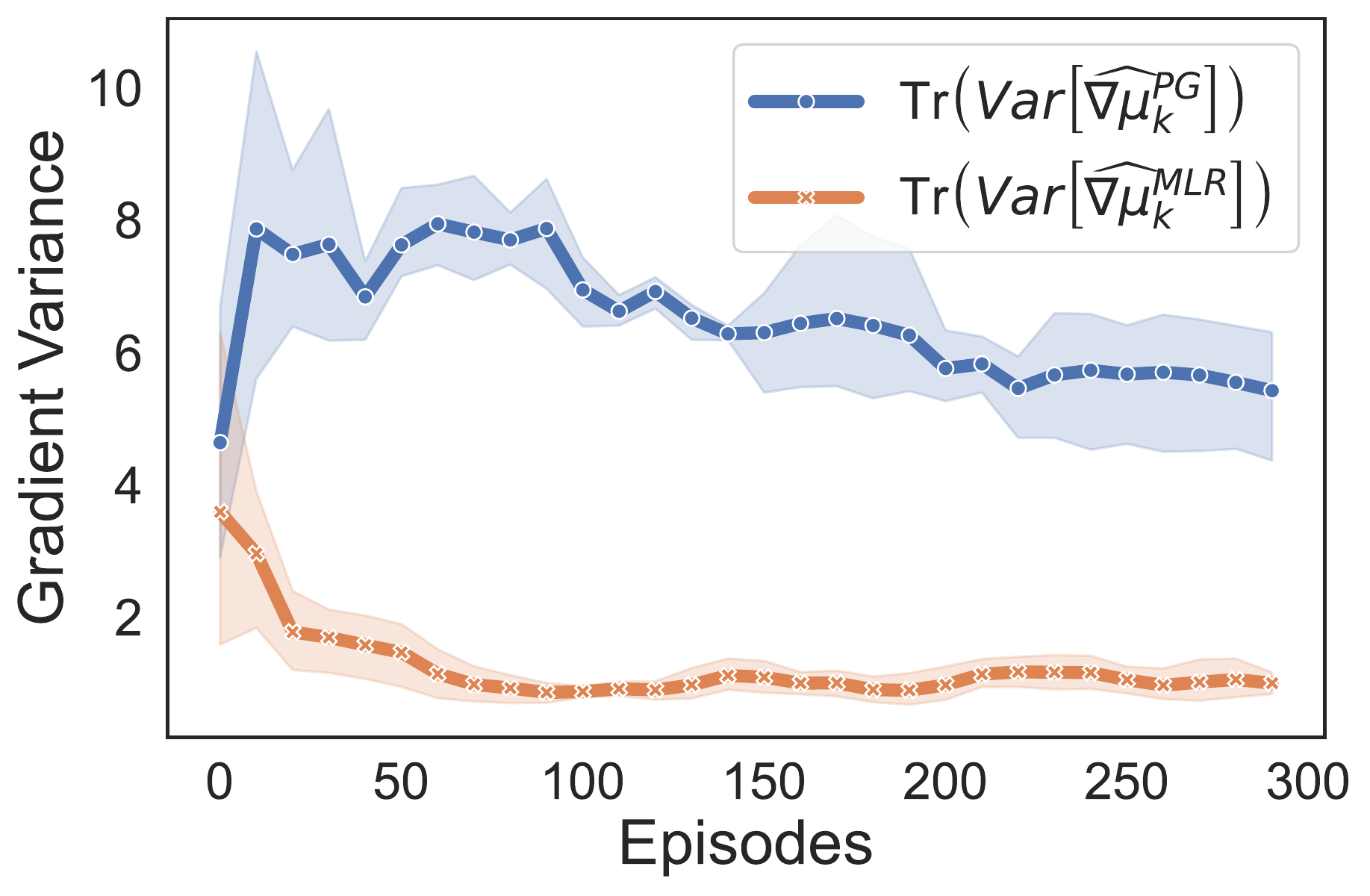}
         \label{fig: PPO gradient norm (CartPole)}}\hspace{-0.02in}%
     \subfloat[Acrobot]{
         \centering
         \includegraphics[width=0.325\textwidth]{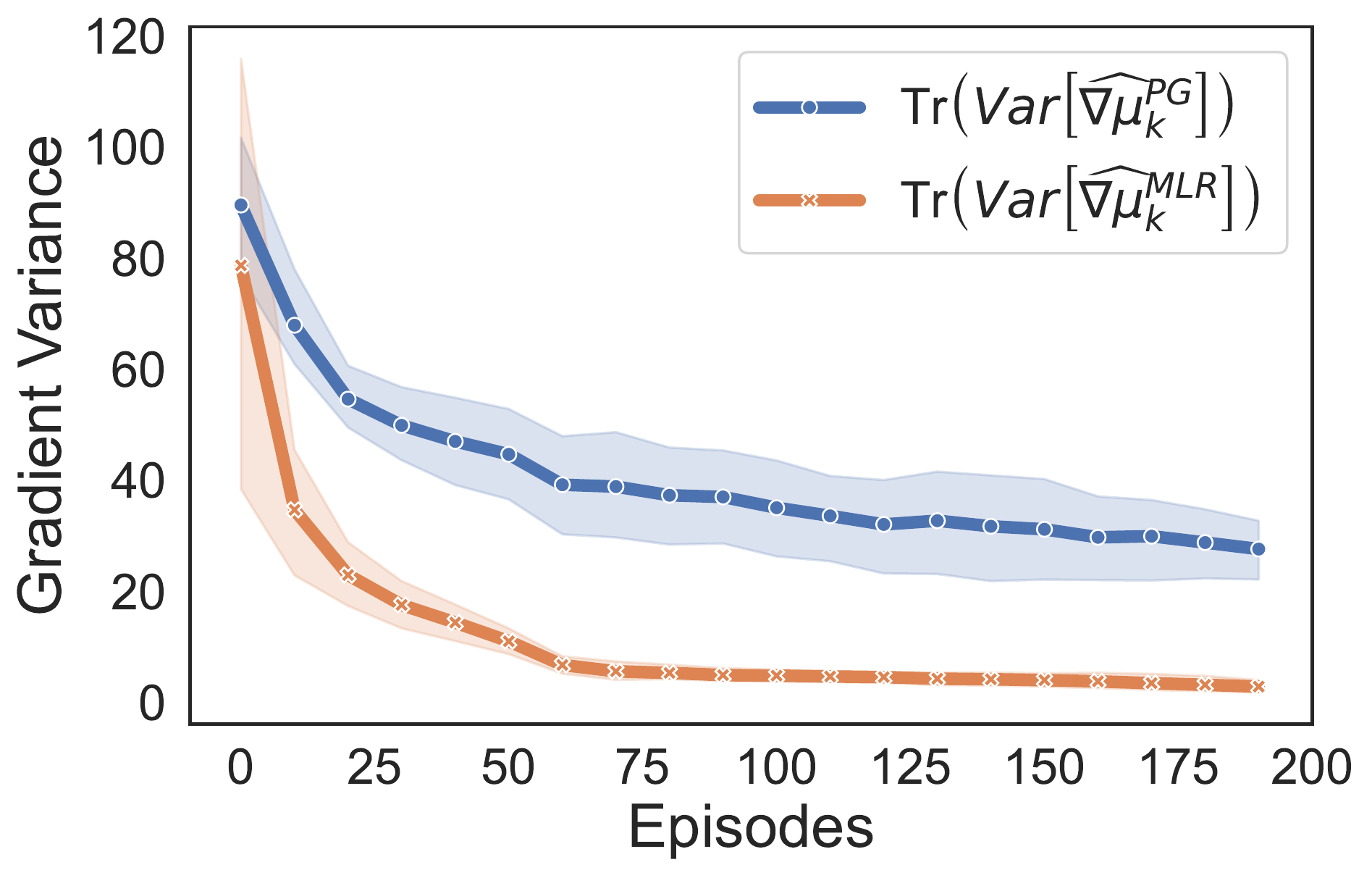}
         \label{fig: PPO gradient norm (Acrobot)}
     }
     \subfloat[Fermentation]{
         \centering
         \includegraphics[width=0.31\textwidth]{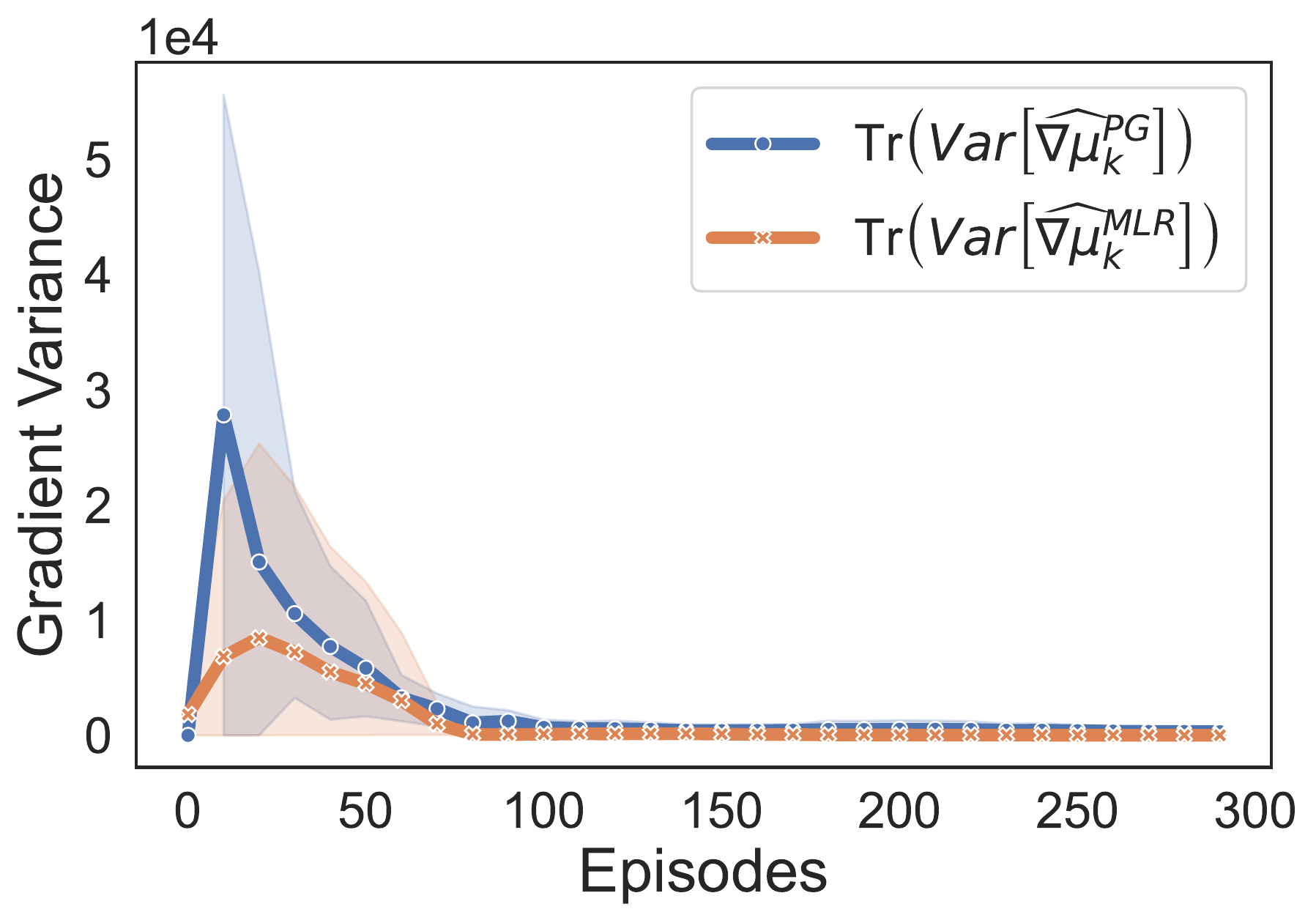}
         \label{fig: PPO gradient norm (Fermentation)}
     }
     \medskip
     \vspace{-0.1in}
     \caption{Policy gradient estimation variance results of  PPO algorithm with and without VRER. %Result are smoothed by the moving average with a smoothing window width of 100.
     } 
     \label{fig: ppo gradient norm}
     \vspace{-0.2in}
\end{figure}

\section{Conclusion}
To guide real-time process control in low-data situations, we create a variance reduction experience replay approach %based partial trajectory reuse 
to accelerate policy gradient optimization.
%In the proposed variance reduced experience replay, 
The proposed selection rule guarantees the variance reduction in the policy gradient estimation through selectively reusing the most relevant historical transition observations and automatically allocating more weights to those observations or partial trajectories 
that are more likely generated by the target stochastic decision process model. The empirical studies show that the incorporation of proposed VRER and MLR with the state-of-the-art policy optimization approaches can substantially improve their optimization convergence especially under the situations with a tight budget.

\bibliographystyle{unsrtnat}
\bibliography{references}  %%% Uncomment this line and comment out the ``thebibliography'' section below to use the external .bib file (using bibtex) .

%%% Uncomment this section and comment out the \bibliography{references} line above to use inline references.
% \begin{thebibliography}{1}

% 	\bibitem{kour2014real}
% 	George Kour and Raid Saabne.
% 	\newblock Real-time segmentation of on-line handwritten arabic script.
% 	\newblock In {\em Frontiers in Handwriting Recognition (ICFHR), 2014 14th
% 			International Conference on}, pages 417--422. IEEE, 2014.

% 	\bibitem{kour2014fast}
% 	George Kour and Raid Saabne.
% 	\newblock Fast classification of handwritten on-line arabic characters.
% 	\newblock In {\em Soft Computing and Pattern Recognition (SoCPaR), 2014 6th
% 			International Conference of}, pages 312--318. IEEE, 2014.

% 	\bibitem{hadash2018estimate}
% 	Guy Hadash, Einat Kermany, Boaz Carmeli, Ofer Lavi, George Kour, and Alon
% 	Jacovi.
% 	\newblock Estimate and replace: A novel approach to integrating deep neural
% 	networks with existing applications.
% 	\newblock {\em arXiv preprint arXiv:1804.09028}, 2018.

% \end{thebibliography}

\end{document}